\documentclass[anon,12pt]{jmlr-style} 

\title[On Randomized Algorithms in Online Strategic Classification]{On Randomized Algorithms in Online Strategic Classification}
\usepackage{times}
\usepackage{tikz}
\usepackage{amssymb}
\usetikzlibrary{calc}
\usepackage{comment}
\usepackage{hyperref}
\usepackage{cleveref}
\usepackage{paralist}
\crefname{definition}{definition}{definitions}
\Crefname{definition}{Definition}{Definitions}

\coltauthor{%
 \Name{Chase Hutton} \Email{chutton6@umd.edu}\\
 \addr University of Maryland
 \AND
 \Name{Adam Melrod} \Email{amelrod@math.harvard.edu}\\
 \addr Harvard University
 \AND
 \Name{Han Shao} \Email{hanshao@umd.edu}\\
 \addr University of Maryland%
}

\newcommand{\mc}{\mathcal}

\newcommand{\Ind}{\mathbf{1}}
\newcommand{\E}{\mathbb E}
\newcommand{\inner}[2]{\left\langle #1,#2\right\rangle}
\DeclareMathOperator{\ldim}{Ldim}

\renewcommand{\paragraph}[1]{\medskip\noindent{\bf #1}\xspace}

\begin{document}

\maketitle

\begin{abstract}%
Online strategic classification studies settings in which agents strategically modify their features to obtain favorable predictions. For example, given a classifier that determines loan approval based on credit scores, applicants may open or close credit cards and bank accounts to obtain a positive prediction. The learning goal is to achieve low mistake or regret bounds despite such strategic behavior.

While randomized algorithms have the potential to offer advantages to the learner in strategic settings, they have been largely underexplored. In the realizable setting, no lower bound is known for randomized algorithms, and existing lower bound constructions for deterministic learners can be circumvented by randomization. In the agnostic setting, the best known regret upper bound is $O(T^{3/4}\log^{1/4}(T|\mathcal H|))$ due to \cite{fundamental-bounds}, which is far from the standard online learning rate of $O(\sqrt{T\log|\mathcal H|})$.

In this work, we provide refined upper and lower bounds for online strategic classification in both the realizable and agnostic settings; our bounds depend on the Littlestone dimension $\ldim(\mc H)$ of the hypothesis class $\mc H$ and the maximum degree $\Delta$ of the manipulation graph. In the realizable setting, using a new construction, we prove a lower bound that, for $T > \ldim(\mc H) \Delta^2$, extends the existing deterministic lower bound of $\Omega(\ldim(\mathcal{H}) \Delta)$ to all algorithms. This is the first lower bound that applies to randomized algorithms, resolving an open question of~\cite{fundamental-bounds}. We also give the first randomized algorithm that improves on the known deterministic upper bound of $O(\ldim(\mathcal H)\cdot\Delta\log\Delta)$, achieving $O(\sqrt{T\cdot\ldim(\mathcal H)\log\Delta})$ expected mistakes, in the regime $T < \ldim(\mathcal H)\,\Delta^2\log\Delta$.

In the agnostic setting, we give an improper randomized algorithm with expected regret \\$O(\sqrt{T\log|\mathcal H|})$ against adaptive adversaries, improving upon the previous $O(T^{3/4}\log^{1/4}(T|\mathcal H|))$ bound and matching the standard online learning rate. We also prove that this optimal rate requires improper learning.
\end{abstract}


\section{Introduction}

Strategic classification lies at the intersection of machine learning and game theory and studies learning problems in which self-interested agents strategically modify their observable features to obtain favorable predictions \citep{bruckner2011stackelberg, hardt2016strategic}. Such behavior arises naturally in many high-stakes decision-making systems, including loan approval, hiring, and college admissions. For example, loan applicants may open or close credit cards to improve their credit scores, while students may retake standardized tests or adjust course selections to increase their chances of admission. This introduces a more complex learning environment, in which agents strategically adapt their features to influence the outcome of the deployed classifier while preserving their true underlying qualifications, such as being successful in college or repaying a loan.

These challenges are further amplified in online environments~\citep{fundamental-bounds}, where a learner repeatedly interacts with a sequence of agents and updates its classifier over time with limited information. In particular, the learner perceives agents only through the lens of the currently deployed hypothesis and does not always have the ability to counterfactually infer how agents would have manipulated under a different hypothesis. As a result, the learner may not know the loss incurred by unimplemented hypotheses and thus cannot obtain full-information feedback.

In this work, we model the set of feasible feature manipulations by a known \emph{manipulation graph} $G=(\mathcal X,E)$. The vertices of $G$ are feature vectors, and an edge $(x,x')\in E$ indicates that an agent with original feature vector $x$ can manipulate to feature vector $x'$.  While much prior work on online strategic classification with manipulation graphs focuses on deterministic algorithms \citep{strat-little-stone-dim, pmlr-v247-cohen24c}, randomized algorithms have been largely underexplored.

As introduced by \cite{fundamental-bounds}, there are two models for randomized learning in strategic settings, depending on whether the learner’s random choice of classifier is revealed to the agent before responding. In the \emph{fractional model}, the learner samples a classifier from a distribution but does not reveal the realization, so the agent best responds to the distribution itself. In the \emph{randomized algorithms model}, the learner reveals the sampled classifier, and the agent best responds to this realization. As shown by \cite{fundamental-bounds}, revealing the realization is strictly more advantageous: when the learner does not reveal the sampled classifier, the learner may suffer linear regret, whereas revealing it allows for sublinear regret. Accordingly, in this work we focus on the randomized algorithms model.

While recent work has established mistake and regret guarantees in several settings \citep{fundamental-bounds, strat-little-stone-dim, pmlr-v247-cohen24c}, several important questions remain open.

One question concerns the role of randomization in the realizable setting. In the non-strategic realizable setting, it is well known that randomization does not improve the mistake bound, and the optimal algorithm is deterministic, namely the Standard Optimal Algorithm (SOA). In contrast, in the strategic setting, prior work \citep{fundamental-bounds, pmlr-v247-cohen24c} has only established lower bounds for deterministic algorithms. Moreover, under the same lower bound construction, \cite{strat-little-stone-dim} shows that randomization can circumvent these lower bounds and significantly improve mistake guarantees. This leads to the following question:
\begin{center}
\textbf{Does randomization always improve mistake bounds in the realizable strategic setting?}
\end{center}

A second question concerns regret bounds in the agnostic setting. In \cite{fundamental-bounds}, the authors provide regret upper bounds of $O(T^{3/4}\log^{1/4}T|\mathcal H|)$ by using an all-positive classifier at randomly chosen time steps to obtain full information, and $O(\sqrt{T|\mathcal H|\log|\mathcal H|})$ by applying standard bandit algorithms such as EXP3. Matching lower bounds, however, are left open. Since it is well known that a regret of $O(\sqrt{T \log|\mathcal H|})$ is achievable in the standard online learning setting, this naturally raises the question:
\begin{center}
\textbf{Can we also achieve a regret guarantee of $O(\sqrt{T \log|\mathcal H|})$ in the agnostic strategic setting?}
\end{center}

In this work, we provide refined upper and lower bounds for online strategic classification in both the realizable and agnostic settings, thereby answering the two questions above.  Our bounds depend on the Littlestone dimension $\ldim(\mc H)$ of the hypothesis class $\mc H$ and the maximum degree $\Delta$ of the manipulation graph. In the realizable setting, we show that every learner incurs at least $\Omega(\min\{\sqrt{T \ldim(\mc H)},\  \ldim(\mc H) \Delta\})$.  
As a consequence, for $T > \ldim(\mathcal{H})\Delta^2$, randomized learners cannot improve upon the standard deterministic lower bound of $\Omega(\ldim(\mathcal{H})\Delta)$, thereby resolving the first question in the negative.
On the other hand, we construct a randomized algorithm that achieves a mistake bound of $O(\sqrt{T \cdot \ldim(\mathcal H) \log \Delta})$, strictly improving upon the best known deterministic upper bounds when $T$ is small, and providing a partial complementary answer to the first question.

In the agnostic setting, we answer the second question in the affirmative: we give an improper algorithm that achieves a regret bound of $O(\sqrt{T \log |\mathcal H|})$, matching the standard (non-strategic) online learning rate up to constants. Since non-strategic instances are a special case, the $\Omega(\sqrt{T \log |\mathcal H|})$ lower bound of \cite{ben2009agnostic} carries over, so this rate is optimal. We further show that improperness is necessary to attain it: we prove a lower bound of $\Omega(\min\{\sqrt{T \log |\mathcal H|} + |\mathcal H|,\ \sqrt{T |\mathcal H| \log |\mathcal H|}\})$ for all proper learners, which exceeds the optimal rate whenever $T < |\mathcal H|^2 / \log |\mathcal H|$.

\subsection{Related work}
Strategic classification was first studied in the distributional setting by \cite{hardt2016strategic}. Since then, researchers have considered different models of agents' behavior, learning tasks, and analytical perspectives.

\paragraph{Online strategic classification.}
Our work studies the online strategic learning setting.  The works of \cite{pmlr-v247-cohen24c, strat-little-stone-dim} also focus on improving the regret bounds established by \cite{fundamental-bounds} for online strategic classification. However, both primarily aim to extend results from finite hypothesis classes to infinite hypothesis classes and focus on deterministic algorithms. The only result concerning randomized algorithms in \cite{strat-little-stone-dim} shows that randomization can circumvent a lower bound for deterministic algorithms in a specific example.
\citet{chen2020learning} study online strategic linear classification in a continuous
(geometric) setting where agents may manipulate their features within a bounded-radius
ball. Both their work and ours share the high-level idea of exploiting the manipulation
structure to extract richer-than-bandit feedback. However, the settings and techniques
differ in important ways. \cite{chen2020learning} employs geometric arguments that have no clear analogue in our combinatorial,
graph-based setting.  Moreover, their regret bound is data-dependent, scaling with the
Lebesgue measure of the smallest induced polytope, which can degrade to $\Omega(T)$ in
adversarial instances.  In contrast, our agnostic bound of
$O(\sqrt{T \log \lvert H \rvert})$ is data-independent and sublinear for all instances.
Additionally, the problem of achieving optimal agnostic regret for general hypothesis
classes under manipulation graphs was explicitly left open by \citet{fundamental-bounds}; Theorem ~\ref{thm:improper-upper-bound} resolves this open problem, and the results do not
follow from the framework of \citet{chen2020learning}.

\paragraph{Randomized models.}
Two notions of randomized models have been studied in the strategic setting. In the \emph{randomized algorithms} model~\citep{chen2020learning, dong2018strategic}, the learner reveals the sampled classifier to the agent, who best responds to the realization. In the \emph{fractional} model~\citep{fundamental-bounds}, the agent best responds to the distribution itself. \citet{braverman2020role} study strategic classification with randomized classifiers in a one-dimensional continuous setting; their model aligns with the fractional framework of \citet{fundamental-bounds}, under which \citeauthor{fundamental-bounds} show that the learner may suffer linear regret. Our work operates in the randomized algorithms model, where revealing the realization is strictly more advantageous and sublinear regret is achievable.
 
\paragraph{Agents' manipulation structure.}
Agents often face restrictions when manipulating their features. 
\cite{hardt2016strategic} modeled manipulation restrictions through a cost function, where agents manipulate to another feature vector if and only if the benefit of obtaining a positive prediction outweighs the cost of manipulation. This formulation naturally leads to two common ways of modeling manipulation structures.
The first is a geometric model (e.g., \cite{dong2018strategic, sundaram2021pac, chen2020learning, ghalme21, haghtalab2020maximizing, ahmadi2021strategic, shao2023strategic}), where agents are allowed to manipulate their features within a bounded-radius ball under a known metric. In this work, following \cite{zhang2021incentive, lechner2022learning, lechner2023strategic, fundamental-bounds}, we adopt an alternative model based on manipulation graphs.
These two modeling approaches can be translated into each other. However, ball-based models typically restrict manipulation through bounded radii, whereas manipulation graphs restrict manipulation through the degree of the graph.

\paragraph{Known manipulation structure.}
Throughout this work we assume the manipulation graph $G$ is known to the learner.
This assumption plays an analogous role to the known manipulation radius~$\delta$ in
the continuous (ball-based) model of \citet{chen2020learning}, and is standard across
the graph-based strategic classification literature
(e.g., \citealp{zhang2021incentive,lechner2022learning,lechner2023strategic,%
fundamental-bounds,strat-little-stone-dim,pmlr-v247-cohen24c,shaoshould}).
The assumption is natural in applications where feasible manipulations are discrete and
can be enumerated by domain experts. For instance, the set of credit-related actions
available to a loan applicant. When the manipulation structure is unknown,
\citet{shao2023strategic} establish impossibility results showing that the learner's
mistake bound must scale as $\Omega(\lvert \mc H \rvert)$, which precludes the type of
refined bounds we obtain here.

\paragraph{Distributional strategic classification.}
Beyond the online setting, there is a large body of work on strategic learning in the offline distributional setting.
Works such as \cite{hardt2016strategic, lechner2022learning, hu2019disparate, milli2019social, sundaram2021pac, zhang2021incentive, lechner2023strategic, shao2023strategic} study settings in which agents are drawn i.i.d.\ from an underlying distribution. The learning objective in these works is to output a classifier with low strategic population loss using a minimal amount of training data.

\paragraph{Different analytical perspectives.} While our work also focuses on prediction accuracy, many other works have explored additional perspectives on strategic classification, such as encouraging genuine improvements versus discouraging ``gaming'' \citep[e.g.][]{ahmadi2022classification, bechavod2021gaming, haghtalab2020maximizing, Liu20}, understanding causal implications \citep[e.g.][]{miller2020strategic, bechavod2021gaming, shavit2020causal, performative2020}, transparency in decision rules \citep[e.g.][]{bechavod2022information, shaoshould}, and fairness considerations \citep[e.g.][]{hu2019disparate, Liu20, milli2019social}.

\section{Model}

\subsection{Strategic classification}
We consider binary classification in the adversarial online setting. Let $\mathcal{X}$ denote the feature vector space, $\mathcal{Y} = \{-1,+1\}$ denote the binary label space, and $\mathcal{H} \subseteq \mathcal{Y}^{\mathcal{X}}$ denote the hypothesis class. In the strategic setting, an example $(x,y)$, referred to as an \textit{agent} in this context, consists of a pair of feature vector and label.
An agent will try to receive a positive prediction by modifying their feature vector to some reachable feature vectors. Following prior work (e.g. \cite{fundamental-bounds,pmlr-v247-cohen24c,strat-little-stone-dim}), we model the set of feasible manipulations using a \emph{manipulation graph} $G=(\mathcal{X}, \mathcal{E})$. The vertices represent feature vectors, and an edge $(x,x') \in \mathcal{E}$ exists if and only if an agent with feature vector $x$ can manipulate to $x'$. We let $N_G(x) = \{x' : (x,x') \in \mathcal{E}\}$ denote the neighborhood of $x$, and $N_G[x] = \{x\} \cup N_G(x)$ denote the closed neighborhood.

Given a manipulation graph $G$, an agent with feature vector $x \in \mathcal{X}$ best responds to a classifier $h$ by moving to a reachable feature vector that is labeled positive, if one exists. More specifically, the agent remains at its original feature vector if it cannot obtain a positive prediction through manipulation or if it is already classified as positive. Otherwise, the agent moves to a reachable feature vector that yields a positive prediction.
Formally, we define the notion of best responses as follows.
\begin{definition}[Best responses]\label{def:br}
Given a manipulation graph $G$, a hypothesis $h \in \mathcal{Y}^{\mathcal{X}}$, and a feature vector $x \in \mathcal{X}$, the set of \emph{best responses} of $x$ under $h$ is
\[
\mathrm{BR}_{G,h}(x) := \{x' \in N_G[x] : h(x') = +1\}.
\]
An agent with feature vector $x$ selects a manipulated feature vector $z$ as follows: if $\mathrm{BR}_{G,h}(x) = \emptyset$ or $x \in \mathrm{BR}_{G,h}(x)$, then $z = x$; otherwise, $z \in \mathrm{BR}_{G,h}(x)$ is chosen according to a fixed tie-breaking rule.
\end{definition}

When implementing a hypothesis $h$, the loss incurred at the agent $(x,y)$ is the misclassification error when the agent manipulates w.r.t.\ $h$. We define the strategic loss as follows.
\begin{definition}[Strategic loss]
Given a manipulation graph $G$, a hypothesis $h \in \mathcal{Y}^{\mathcal{X}}$, and an agent $(x,y)$, let $z$ denote the manipulated feature vector chosen by the agent in response to $h$. The \emph{strategic loss} of $h$ at $(x,y)$ is defined as
\[
\ell_G^{\mathrm{str}}(h,(x,y)) := \Ind\{h(z) \neq y\}=
\begin{cases}
\Ind\{h(x) \neq y\}, & \text{if } \mathrm{BR}_{G,h}(x)=\emptyset \text{ or } x \in \mathrm{BR}_{G,h}(x),\\
\Ind\{y=-1\}, & \text{otherwise}.
\end{cases}.
\]
\end{definition}
Note that strategic behavior only harms the learner on negative instances ($y=-1$). If $y=+1$ and the agent manipulates to a positive classification, the prediction is correct (loss 0). However, if $y=-1$ and the agent successfully manipulates (i.e., $\mathrm{BR}_{G,h}(x) \neq \emptyset$), the learner predicts positive while the true label is negative, incurring loss $1$.

\subsection{Online learning with strategic agents}
The learner classifies a sequence of strategic agents ${(x_t,y_t)}_{t=1}^{T}$ selected by an adversary. At each round, the learner selects a distribution over hypotheses and samples a hypothesis from this distribution. As discussed in the introduction, the learner then reveals the sampled hypothesis to the agent, who best responds after observing it. At the end of each round, the learner observes the true label as well as only the manipulated features, rather than the agent’s original features (e.g., the learner observes post-manipulation credit scores in loan approval or the final SAT scores submitted by student applicants).

More formally, at each round $t$, the learner selects a distribution $\mathcal D_t$, the adversary observes $\mathcal D_t$, and then selects an agent $(x_t,y_t)$. The learner then samples a hypothesis $h_t \sim \mathcal{D}_t$ and reveals $h_t$ to the agent. After observing $h_t$, the agent best responds by manipulating her feature vector from $x_t$ to $z_t$ according to Definition~\ref{def:br}. At the end of the round, the learner observes the prediction on the manipulated feature, $\hat y_t = h_t(z_t)$, as well as the true label $y_t$, and incurs a loss whenever $\hat y_t \neq y_t$. This interaction protocol, which repeats for $t=1,\ldots,T$, is summarized in Protocol~\ref{prot:interaction}.

The learner's goal is to minimize the expected \emph{Stackelberg regret} against the best fixed strategy in hindsight:
\[
    \E[\mathfrak{R}_T] \;=\; \E \left[\sum_{t=1}^T \ell_G^{\mathrm{str}}(h_t, (x_t, y_t)) -  \min_{h \in \mc H}\sum_{t=1}^T\ell_G^{\mathrm{str}}(h, (x_t, y_t)) \right],
\]
where the expectation is taken over the learner's internal randomization. When the best hypothesis in $\mathcal H$ has zero loss, i.e., $\min_{h\in \mathcal H} \sum_{t=1}^T\ell_G^{\mathrm{str}}(h, (x_t, y_t))=0$, we call the sequence of agents \emph{realizable} and otherwise call it \emph{agnostic}.

\begin{algorithm2e}[ht]
\SetAlgorithmName{Protocol}{protocol}{List of Protocols}
\caption{Online Strategic Classification}
\label{prot:interaction}
\DontPrintSemicolon
\KwIn{Feature space $\mathcal{X}$, Label space $\mathcal{Y}$, Hypothesis class $\mathcal{H}$, Graph $G$, Horizon $T$.}
\For{$t = 1, \dots, T$}{
    The learner selects a distribution $\mathcal{D}_t \in \Delta(\mc Y^{\mc X})$.\;
    The adversary selects an agent $(x_t, y_t) \in \mathcal{X} \times \mathcal{Y}$.\;
    The learner samples $h_t \sim \mathcal{D}_t$ and reveals $h_t$ to the agent.\;
    The agent manipulates to a feature vector $z_t$ following the manipulation rule in  Definition~\ref{def:br}.\;
    The learner predicts $\hat{y}_t = h_t(z_t)$.\;
    The learner observes the manipulated features $z_t$ and true label $y_t$ and suffers loss $\ell_t = \Ind\{\hat{y}_t \neq y_t\}$.\;
}
\end{algorithm2e}
\section{Overview of Results}

\subsection{Realizable setting}\label{sec:realizable}

In this section, we derive upper and lower expected mistake bounds for the realizable setting in both the cases of finite and infinite hypothesis classes $\mathcal H$. For finite classes, we derive guarantees in terms of the size $n$ of the hypothesis class, while for infinite classes, our bounds are given in terms of the Littlestone dimension $d$ of $\mathcal H$ and the max degree $\Delta$ of the manipulation graph $G$.

\subsubsection{Upper Bounds}
In this section we give two algorithms, the first of which works for finite hypothesis classes and the second of which works for infinite hypothesis classes. Both algorithms work at a high level by maintaining a system of weighted experts. Each round, the algorithm decides to either choose an expert randomly by weight for use as a classifier, or to use an all-positive classifier to obtain full-information which is used to update the expert system.
This decision is made randomly by coin flip, biased to ensure the decrease in weight of the system is correlated to the chance of an expert making a mistake. By understanding how the weight changes, we may then upper bound the number of mistakes. Our algorithmic result for finite classes is the following.
\begin{theorem}\label{thm:finite-realizable-upper-bound}
    For any graph $G$ and hypothesis class $\mathcal H$ of size $n$, and for any realizable sequence, we can achieve expected mistake bound $O(\sqrt{T\log n})$ using Algorithm \ref{alg:mix-realizable}. By combining this algorithm with the naive deterministic algorithm that tries every hypothesis until it fails, we obtain a learner with expected mistake bound $O(\min \{\sqrt{T\log n}, n\})$.
\end{theorem}
Our algorithm is based on the following intuition. Say $(x_t,y_t)$ is the agent for round $t$. Let us consider two strategies. In the first strategy the learner commits to sampling a hypothesis uniformly at random from the set of remaining consistent hypothesis $V_t$. Then the probability that a mistake occurs is exactly the proportion $\delta_t$ of hypotheses in $V_t$ which make a mistake.
Next consider the strategy where the learner commits to using an all positive classifier $h^+$ which labels every vertex $+1$. Then since $x_t$ is predicted as $+1$, the agent will never manipulate. By removing each hypothesis that made a mistake, we can ensure that $|V_{t+1}| = (1-\delta_t)|V_t|$. The key observation is that when $\delta_t$ is \emph{small}, sampling a random $h\in V_t$ leads to a small mistake probability. When $\delta_t$ is \emph{large}, always predicting $+1$ leads to a large reduction of the hypothesis set. By playing an appropriate random mix of these two strategies we can ensure (in expectation) that we don't suffer both a high mistake probability and poor learning progress on the same round. The pseudocode is given in Algorithm \ref{alg:mix-realizable}. We conclude by sketching the proof of Theorem \ref{thm:finite-realizable-upper-bound}.

\begin{algorithm2e}[ht]
\caption{Uniform-Mix}
\label{alg:mix-realizable}
\DontPrintSemicolon

Set $p \leftarrow \min\!\left\{1,\sqrt{{\log n}/{T}}\right\}$,
$V_1 \leftarrow \mathcal H$.\;

\For{$t \leftarrow 1$ \KwTo $T$}{
  Sample $C_t \sim \mathrm{Ber}(p)$.\;

  \lIf{$C_t = 1$}{
    $h_t \leftarrow h^{+}$,
  }
  \lElse{
    sample $h_t \sim \mathrm{Unif}(V_t)$.
  }
  Play $h_t$ and receive $(z_t,y_t)$.\;
  \lIf{$h_t = h^+$}{
   $V_{t+1} \gets \{h \in V_t : \ell_G^{\text{str}}(h, (z_t,y_t)) = 0\}.$
  }
}
\end{algorithm2e}

\paragraph{Proof (sketch).} We split mistakes into two types: those incurred by the all positive strategy (exploration mistakes) and those incurred by the uniformly random strategy (expert mistakes). The exploration mistakes are upper bounded by the number of times the all positive strategy is selected, which in expectation is $pT$. We understand the expert mistakes by relating the mistake probability at round $t$ to the change in valid hypotheses $V_t$. In particular we can show the following relationship.
\[
    \Pr[\text{expert mistake on round } t] = (1-p)\delta_t \leq \frac{1-p}{p} \E\left[\log \frac{|V_{t}|}{|V_{t+1}|}\right].
\]
After summing over all rounds $t$ and combining with potential exploration mistakes, we obtain  $\E[M(T)] \leq pT + \frac{1-p}{p} \log |V_1|$. Setting $p = \min\{1,\sqrt{{\log n}/{T}}\}$ and using $V_1 = \mathcal H$, we obtain $\E[M(T)] = O\left(\sqrt{T \log n}\right)$.

Next, we give an algorithm that works for infinite hypothesis classes, which yields the following.
\begin{theorem}\label{thm:infinite-realizable-upper-bound}
    For any graph $G$ and hypothesis class $\mathcal H$, and for any realizable sequence, we can achieve expected mistake bound $O(\sqrt{T d \log \Delta})$ using Algorithm \ref{alg:mix-expert}, where $d = \text{Ldim}(\mathcal{H})$ and $\Delta$ is the max degree of $G$.
\end{theorem}

Our algorithm will take as input a standard online learning algorithm $\mathcal A$ that has mistake bound $M_{\mathcal A}$. In our case this will be the SOA algorithm by \cite{littlestone1988learning} which has $M_{\mathcal A}(T) = O(d)$. The algorithm will maintain a set of experts $\mathcal E_t$ for each round $t$, constructed by feeding various sequences of agents into the algorithm $\mathcal A$. If $E$ is an expert, denote by $E(x,y)$ the expert obtained by giving $\mathcal A$ the agent sequence of $E$ followed by $(x,y)$. Each expert $E \in \mathcal E_t$ will have an associated weight $w_t(E)$. We denote the total weight at round $t$ by $W_t := \sum_{E \in \mathcal E_t} w_t(E)$.

As in the finite case, for each round, there are two possible procedures chosen by a biased coin flip. The first procedure is an exploration round, where an all-positive classifier is selected. As using such a classifier results in full information, we can identify all incorrect experts, and update our expert system accordingly (this update is identical to that used in Algorithm 2 from \cite{pmlr-v247-cohen24c}). The second procedure is to act according to the expert system, so an expert will be randomly sampled with probability $w_t(E)/W_t$ to be used as a classifier.

There are three fundamental behaviors that determine how we must update the expert system. First, the total weight associated to incorrect experts at round $t$ must decrease by a factor of $1/2$. Second, the sequence associated to an expert must consist of only mistake rounds for the online learner. Third, there must be an expert whose associated sequence is realizable. These behaviors are important so as to link the standard online algorithm mistake bound, the change in weight in a round, and the efficiency of the experts. The pseudocode is given in section \ref{sec:pseudocode-for-infinite} of the appendix.

We enforce this behavior by updating experts in the following way. If $y_t = -1$, then replace every expert $E$ that labels some neighbor of $x_t$ positive, say $x$, by the expert $E(x,-1)$ possessing half the weight of $E$. Else $y_t = +1$, so replace each expert $E$ that labels all of $N_G[x_t]$ negative by $|N_G[x_t]|$ experts $E(x,+1)$, where $x$ ranges over $N_G[x_t]$, each possessing weight $w_t(E)/2|N_G[x_t]|$. The asymmetry in handling the different values of $y_t$ derives from the asymmetry of the strategic learning feedback model; specifically, the condition to ensure both a realizable sequence exists and each agent makes a mistake changes depending on the value of $y_t$.

\paragraph{Proof.} As in the finite case, we split mistakes into exploration mistakes and expert mistakes. The first step is to relate the change in weight at round $t$ to the chance of an expert mistake. Some analysis (Lemma \ref{lem:exp-mistake-prob-inf} in Appendix \ref{sec:appendix-A}) shows that
    for all rounds $t$, we have
    \[
        \Pr[\text{expert mistake at round } t] \leq \frac{2(1-p)}{p} \cdot \E\left[\log\frac{W_t}{W_{t+1}}\right].
    \]
Having established this relation, and combining exploration and expert mistakes, we conclude that
\[
    \E[M(T)] \leq pT + \frac{2(1-p)}{p} \cdot \E\left[\log\frac{W_1}{W_{T+1}}\right].
\]
To grasp the total change in weight from the first to the $T$th round, we take advantage of how the experts are constructed; every expert corresponds to $\mathcal A$ being fed an associated sequence of agents which all force mistakes. By how the weights are allocated, any given expert has weight at least $(1/2\Delta)^\ell$, where $\ell$ is the length of the agent sequence associated to that expert. Moreover, we can show that since the actual agent sequence is realizable, there is always at least one expert $E$ whose associated agent sequence is both realizable and forces mistakes each round when fed into $\mathcal A$ (Lemma \ref{lem:realizable-expert-exists} from Appendix \ref{sec:appendix-A}). For that sequence, we thus have $\ell \leq M_{\mathcal A}$, and so $W_{T+1} \geq w_{T+1}(E) \geq (1/2\Delta)^{M_{\mathcal A}}$. As $W_1 = 1$, we conclude that
\[
    \E[M(T)] \leq pT + \frac{2(1-p)}{p} \cdot M_{\mathcal A} \log 2\Delta.
\]
Taking $p = \min\!\left\{1,\sqrt{M_{\mathcal A}\log 2\Delta / T}\right\}$ and using $M_{\mathcal A} = d$ completes the proof.

By combining this randomized algorithm with the deterministic Algorithm 2 from \cite{pmlr-v247-cohen24c}, we obtain the following corollary.
\begin{corollary}
    For any infinite class $\mathcal H$ with Littlestone dimension $d$ and manipulation graph $G$ with max degree $\Delta$, there is a learner such that for any realizable adversary, the expected number of mistakes is at most $O(\min\{\sqrt{T d \log \Delta}, d \cdot \Delta \log \Delta\})$.
\end{corollary}

\subsubsection{Lower Bounds}\label{sec:realizable-LB}
In this section, we shift our focus to obtaining lower bounds for randomized learners. The central idea will be to construct a manipulation graph and hypothesis class so that the learning interaction is essentially a hidden index game. Our main result is the following.
\begin{theorem}\label{thm:realizable-LB}
    For each $n$, there is a graph $G$ with hypothesis class $\mathcal H$ of size $n$, and a realizable adversary such that for any randomized learner, the expected mistakes are $\Omega(\min\{n,\sqrt{T}\})$.
\end{theorem}

In the above construction, we will have $\text{Ldim}(\mc H) = 1$ and $n = \Delta$. We can extend the construction to the case of $\text{Ldim}(\mc H) = d$ by making $d$ independent copies. Thus, we derive the following corollary whose proof is deferred to Appendix \ref{sec:appendix-A-proof-of-infinite-lb}.

\begin{corollary}\label{cor:realizable-lb-inf}
    For each $n$, there is a graph $G$ with max degree $\Delta$ and hypothesis class $\mathcal H$ of Littlestone dimension $d$, and a realizable adversary such that for any randomized learner, the expected mistakes are $\Omega(\min\{\sqrt{Td},d\Delta\})$.
\end{corollary}

\paragraph{Construction.}
Build an undirected manipulation graph with two layers
$U=\{u_1,\dots,u_n\}$ and $P=\{p_1,\dots,p_n\}$, together with three additional
vertices $L,R,z$. The edges are arranged so that each $u_i$ connects to all $p_j$
\emph{except} $p_i$; additionally, $L$ connects to all of $U$, $R$ connects to all of $P$,
and $z$ connects only to $R$. The hypothesis class $\mathcal H$ is the family of `singletons over $P$':
for each $i\in[n]$, hypothesis $h^i$ labels only $p_i$ as $+1$ and labels every other
vertex as $-1$. 
Best responses break ties toward the $p_i$ with the smallest index.
This construction is depicted in Figure \ref{fig:graph-cons}.
\begin{figure}[ht]
    \centering
\begin{tikzpicture}[
    every node/.style={circle, draw, minimum size=9mm, inner sep=0pt, font=\small},
    >=stealth,
]
\path[use as bounding box] (-0.5,-6.9) rectangle (8.5,0.6);
\node (u1)     at (0, 0)      {$u_1$};
\node (udots1) at (0, -1.4)   {$\vdots$};
\node (uk)     at (0, -2.8)   {$u_k$};
\node (udots2) at (0, -4.2)   {$\vdots$};
\node (um)     at (0, -5.6)   {$u_n$};

\node (p1)     at (4, 0)      {$p_1$};
\node (pdots1) at (4, -1.4)   {$\vdots$};
\node (pk)     at (4, -2.8)   {$p_k$};
\node (pdots2) at (4, -4.2)   {$\vdots$};
\node (pm)     at (4, -5.6)   {$p_n$};

\node (tU)  at (-2.0, -2.8) {$L$};
\node (tP)  at (7.0, -2.8)  {$R$};
\node (z)  at (9.0, -2.8)  {$z$};

\draw[-] (u1) -- (pk);
\draw[-] (u1) -- (pm);

\draw[-] (uk) -- (p1);
\draw[-] (uk) -- (pm);

\draw[-] (um) -- (p1);
\draw[-] (um) -- (pk);
\draw[-] (tP) -- (p1);
\draw[-] (tP) -- (pk);
\draw[-] (tP) -- (pm);

\draw[-] (tU) -- (u1);
\draw[-] (tU) -- (uk);
\draw[-] (tU) -- (um);

\draw[-] (z) -- (tP);

\draw[ dashed] (u1) -- (p1);
\draw[ dashed] (uk) -- (pk);
\draw[ dashed] (um) -- (pm);

\node[draw=none, font=\small] at (-0.9, 0.2) {$U$};
\node[draw=none, font=\small] at (4.9, 0.2) {$P$};

\node[draw=none, font=\small] at ($(u1)+(-0.45,0.45)$) {$-$};
\node[draw=none, font=\small] at ($(uk)+(-0.45,0.45)$) {$-$};
\node[draw=none, font=\small] at ($(um)+(-0.45,0.45)$) {$-$};

\node[draw=none, font=\small] at ($(p1)+(0.45,0.45)$) {$-$};
\node[draw=none, font=\small] at ($(pm)+(0.45,0.45)$) {$-$};

\node[draw=none, font=\small] at ($(pk)+(0.45,0.45)$) {$+$};

\node[draw=none, font=\small] at ($(tU)+(-0.45,0.45)$) {$-$};
\node[draw=none, font=\small] at ($(tP)+(0.45,0.45)$) {$-$};
\node[draw=none, font=\small] at ($(z)+(0.45,0.45)$) {$-$};

\node[draw=none, font=\small, align=center] at (3.5, -6.5)
  {Hypothesis $h_k$ labels $p_k$ as $+1$ and all other nodes as $-1$.};

\end{tikzpicture}
    \caption{Manipulation graph for lower bound.}
    \label{fig:graph-cons}
\end{figure}

\paragraph{Adversary for proper learners.} We first present an adversary for proper learners. After, we will adapt this strategy to work for all learners. At the beginning of the interaction, the adversary samples a hidden index $i^* \sim \mathrm{Unif}([n])$. Each round the adversary plays the agent $(u_{i^*},-1)$, which is realizable by the hypothesis $h_{i^*}$. Note also that if a learner plays $h_{i}$ for any $i \neq i^*$, then the learner incurs a mistake. This strategy reduces the interaction to the learner attempting to determine the hidden index. Each round the learner either guesses correctly, in which case the learner now knows $i^*$ so no more mistakes will be made after that point, or the learner guesses incorrectly and rules out a single index from consideration. As the hidden index is chosen uniformly at random, one can deduce that the probability the learner does not find $i^*$ within $t$ rounds is at least $(n-t)/n$. Taking $t = \min\{T,n/2\}$ shows that any learner incurs at least $\min\{T/2,n/4\}$ mistakes in expectation.

\paragraph{Adversary for all learners.} As in the proper case, at the start of the interaction, the adversary samples a hidden index $i^* \sim \mathrm{Unif}([n])$, which determines the
realizable target $h_{i^*}$. At each round $t$, the adversary adaptively selects agents based on the learner's
announced distribution $D_t$ over classifiers.
We define three `bad' sets of classifiers:
\begin{align*}
  A &:= \{f: f(L)=+1 \ \text{or}\ \exists i\in[n]\ \text{s.t. } f(u_i)=+1\}, \\
  B &:= \{f: f(R)=+1 \ \text{or}\ f(z)=+1\},\\
  C &:= \{f: f(P \cup \{R,z\})=-1\}.
\end{align*}
These sets are chosen so that if a classifier is a member, it makes a mistake on a particular choice of agent whereas a classifier not in any of these sets (including all the hypotheses in $\mathcal H$) do not. The agent witnessing this is $(L,-1)$ for $A$, $(z,-1)$ for $B$, and $(R,+1)$ for $C$.

We call these sets `bad' since the classifiers in the complement, the `good' set, behave almost the same as hypotheses from $\mathcal H$. In the case that a learner plays a 'good' hypothesis, the adversary should play $(u_{i^*},-1)$ according to the proper learner strategy, and in the case that a `bad' hypothesis is played, the adversary should punish the learner by choosing an agent according to the paragraph above. This informs the following strategy.
Let $\gamma = \Theta(1/\sqrt{T})$. On round $t$, the adversary plays according to the following cases.
\begin{inparaenum}[(1)]
  \item If $\mathcal{D}_t(A) \ge \gamma$, play $(L,-1)$.
  \item Else if $\mathcal{D}_t(B) \ge \gamma$, play $(z,-1)$.
  \item Else if $\mathcal{D}_t(C) \ge \gamma$, play $(R,+1)$.
  \item Otherwise, play $(u_{i^*},-1)$.
\end{inparaenum}
Any agent sequence resulting from this strategy is realizable by $h_{i^*}$.
The given adversary on the constructed graph witnesses Theorem \ref{thm:realizable-LB}.

\paragraph{Proof (sketch).}
Cases 1--3 are designed so that whenever the adversary triggers one of them, the learner
incurs mistake probability at least $\gamma$ on that round (because $\mathcal{D}_t$ places at least
$\gamma$ mass on classifiers that are forced to err in that case). Hence, if there are fewer
than $T/2$ case 4 rounds, then there are at least $T/2$ case 1--3 rounds and
$\E[M(T)] \ge \frac{\gamma T}{2} = \Omega\left(\min\{n,\sqrt{T}\}\right)$.

Otherwise, analyze the first $\tau$ case 4 rounds. If in each of these rounds, the hypothesis $h$ sampled is good, so $h$ is in $(A \cup B \cup C)^c$, then similar reasoning as in the proper case holds to conclude that the probability of incurring a mistake in all of those rounds is at least $(n-\tau)/n$. Now, in any case 4 round, the chance that the hypothesis sampled is good is at least $1-3\gamma$, so we determine that overall, the probability of incurring a mistake in each of the first $\tau$ case 4 rounds is at least $(1-3\gamma)^\tau \cdot \frac{n-\tau}{n} \geq \frac{1-3\gamma\tau}{2}$.
The inequality follows from $\tau \le n/2$ and Bernoulli's inequality. By taking $\tau = \Theta(\min\{n,\sqrt{T}\})$,  we obtain $
    \E[M(T)] = \Omega\left(\min\{n,\sqrt{T}\}\right).$

\subsection{Agnostic setting}\label{sec:agnostic} We present our main results for the agnostic setting with finite hypothesis classes. First, we give an improper online algorithm achieving expected regret $O(\sqrt{T \log |\mc H|})$. Since non-strategic instances are a special case of strategic learning, the $\Omega(\sqrt{T \log |\mc H|})$ lower bound from \cite{ben2009agnostic} carries over, and our upper bound is tight. Second, we show that improper learning is necessary to attain the optimal rate, by providing a lower bound of $\Omega(\min\{\sqrt{T \log |\mc H|} + |\mc H|, \sqrt{T|\mc H|\log|\mc H|}\})$ for any proper learner. 

\subsubsection{Upper Bound}
In this section, we prove the following theorem.

\begin{theorem}\label{thm:improper-upper-bound}
    Let $\mc H$ be a finite hypothesis class of size $n$ and $G$ a manipulation graph. Then there exists an improper algorithm which guarantees an expected regret bound of $O(\sqrt{T \log n})$ against any adaptive adversary.
\end{theorem}

\paragraph{Feedback structure.}
As observed by \cite{fundamental-bounds}, the feedback in strategic classification lies between full-information and bandit feedback. At the end of each round the learner observes the manipulated feature $z_t$ and the true label $y_t$, but not the original feature $x_t$. However, in two situations the learner can certify that no manipulation occurred, that is $z_t = x_t$.
\begin{enumerate}[(i)]
    \item When $h_t(z_t) = -1$, since an agent manipulates only to secure a positive prediction, a negative prediction at $z_t$ implies no manipulation took place; or
    \item When $h_t(z_t) = +1$ and $h_t(x') = +1$ for every $x' \in N_G[z_t]$, since every vertex in the closed neighborhood of $z_t$ is already classified positively, any potential original feature $x_t \in N_G[z_t]$ was already positive, so the agent had no incentive to move.
\end{enumerate}
In either case, knowing $z_t = x_t$ reveals the closed neighborhood $N_G[x_t]$, from which the strategic loss of every hypothesis in $\mc H$ can be computed. We refer to any round on which one of these conditions holds as a \emph{reveal event}.

\paragraph{Prior work.}
The algorithm of \cite{fundamental-bounds} exploits this structure by employing an exponential weights algorithm over $\mc H$ with updates prompted by periodic exploration via the all-positive classifier $h^+$. The reveal event in this case is certified by situation (ii) above. On each round, with probability $p$ the learner plays $h^+$ to obtain full information and update the weights; otherwise the learner samples from the current distribution. The learner therefore pays an explicit exploration cost of $pT$ from playing $h^+$, and a cost of order $\sqrt{T\log |\mc H|}/p$ from running Hedge on importance-weighted estimators with loss that scales as $1/p$. Balancing these terms yields a regret of $O(T^{3/4}\log^{1/4}(T|\mc H|))$.

\paragraph{Our approach.}
We improve upon the results of \cite{fundamental-bounds} by using reveal events more efficiently. Like their approach, our algorithm mixes in the all-positive classifier $h^+$, but $h^+$ is no longer the only source of full-information updates. Instead, the learner maintains exponential weights over $\mc H$, plays $h^+$ with probability $\rho$, and otherwise samples $h_t\sim p_t$ from the current weights. Whenever the resulting round yields a reveal event, either because $h_t=h^+$ or because the sampled proper hypothesis satisfies $h_t(z_t)=-1$, the learner computes a loss estimator $\hat d_t$ and updates the weights accordingly. Formally, the learner maintains weights $w_t(i)$ over $\mc H$, which induces a distribution $p_t(i) := w_t(i)/W_t$. Upon a reveal event, the weights are updated by $w_{t+1}(i) \propto w_t(i)\exp(-\eta \widehat d_t(i))$, where $\widehat d_t$ estimates an appropriate loss vector $d_t$. We will construct $\widehat d_t$ so that it is bounded, has constant weighted second moment $\E[\langle p_t,\widehat d_t^{\,2}\rangle]$, and has controlled bias and uniform deviation from the true cumulative losses. These properties ensure that the Hedge algorithm performs well.

\paragraph{The estimator.} Write $\mc H$ as $\{h^1,h^2,\dots,h^n\}$ where $h^i$ ranges over the $n$ hypothesis of $\mc H$. Let
\[
    R_t := \{h^i \in \mc H : h^i(x') = -1 \text{ for all } x' \in N_G[x_t]\},
\]
the set of hypotheses labeling the entire closed neighborhood of $x_t$ as negative. The strategic loss can be expressed in terms of $R_t$:
\[
    \ell_G^{\mathrm{str}}(h^i, (x_t, y_t)) = \begin{cases}
        \Ind\{h^i \in R_t\}, & y_t = +1, \\
        1 - \Ind\{h^i \in R_t\}, & y_t = -1.
    \end{cases}
\]
Rather than estimating this loss vector directly, we estimate an equivalent shifted loss. Specifically, define
\[
    d_t(i)
    :=
    \ell_G^{\mathrm{str}}(h^i,(x_t,y_t))-\Ind\{y_t=-1\}
    =
    y_t\Ind\{h^i\in R_t\}.
\]
The shift is the same for every hypothesis, and therefore cancels in regret comparisons. We adopt an estimator inspired by the implicit exploration technique of \cite{neu2015}. Let $\eta \in (0,1/2)$ be a tunable learning parameter. On each round, with probability $\rho := 2\eta$, the learner plays $h^+$, and otherwise samples $h_t \sim p_t$. Let $A_t := \Ind\{h_t = h^+ \text{ or } h_t(z_t) = -1\}$ indicate whether a reveal event occurred, and let $a_t$ denote its probability of occurring, which is $\rho + (1-\rho)p_t(R_t)$.

We now define the estimator. Let
\begin{equation}
    \widehat d_t(i) := A_t \cdot \frac{y_t \Ind\{h^i \in R_t\}}{a_t + \eta y_t}.
    \label{eq:improper-estimator}
\end{equation}
We verify that this estimator satisfies the three desired properties. Let $\mc F_{t-1}$ denote the history before round $t$, and define $\mc G_t := \sigma(\mc F_{t-1}, x_t, y_t)$. First, since $a_t \ge \rho = 2\eta$, the denominator satisfies $a_t + \eta y_t \ge \eta$ for both labels, so $|\widehat d_t(i)| \le 1/\eta$. Second, since $\widehat d_t(i)$ is nonzero only for $h^i \in R_t$ and only when $A_t$ holds, a direct calculation gives
\[
    \E[\langle p_t,\widehat d_t^{\,2}\rangle \mid \mc G_t]
    = \frac{a_t \cdot p_t(R_t)}{(a_t + \eta y_t)^2} \le 4,
\]
where the bound follows from $p_t(R_t) \le a_t$ and $a_t - \eta \ge a_t/2$. Similarly, the bias satisfies
\[
    \E[\langle p_t,d_t-\widehat d_t\rangle \mid \mc G_t]
    = p_t(R_t) \cdot \frac{\eta y_t^2}{a_t + \eta y_t} \le 2\eta.
\]
Third, the $\eta y_t$ correction in the denominator is chosen so that $\E[\exp(\eta(\widehat d_t(i) - d_t(i))) \mid \mc G_t] \le 1$ for each $i \in [n]$. This makes $\exp(\eta \sum_{s \le t}(\widehat d_s(i) - d_s(i)))$ a supermartingale, from which a standard log-sum-exp argument yields
\[
    \E\!\left[\max_i \sum_{t=1}^T(\widehat d_t(i)-d_t(i))\right]
    \le \frac{\log n}{\eta}.
\]
Thus, the estimator has exactly the properties required.

\begin{algorithm2e}[ht]
\caption{Improper-Hedge}\label{alg:improper}
\DontPrintSemicolon
Initialize $w_1(i) \gets 1$ for all $i \in [n]$\;
\For{$t \gets 1$ \KwTo $T$}{
  Set $p_t(i) \gets w_t(i) / W_t$ where $W_t = \sum_{j=1}^n w_t(j)$\;
  Sample $B_t \sim \mathrm{Bernoulli}(\rho)$\;
  \lIf{$B_t = 1$}{$h_t \gets h^+$}
  \lElse{sample $h_t \sim p_t$}
  Play $h_t$ and receive $(z_t, y_t)$\;
  \If{$h_t = h^+$ or $h_t(z_t) = -1$}{
    Compute $R_t = \{h^i : \forall x' \in N_G[z_t],\ h^i(x') = -1\}$ and $a_t \gets \rho + (1-\rho)p_t(R_t)$\;
    $\widehat d_t(i) \gets \Ind\{h^i \in R_t\} \cdot \frac{y_t}{a_t + \eta y_t}$ for $i \in [n]$\;
  }
  \lElse{$\widehat d_t \gets 0 \in \mathbb{R}^n$}
  Update $w_{t+1}(i) \gets w_t(i) \exp(-\eta \widehat d_t(i))$ for all $i \in [n]$\;
}
\end{algorithm2e}

\paragraph{Remark on proper learning.} Without the improper action $h^+$, one can still construct estimators for the shifted loss using only condition (i) for reveals. In that case, the reveal probability reduces to $p_t(R_t)$, which may be arbitrarily small, making the importance-weighted estimator unbounded. To accommodate this, one can use the FTRL with log-barrier regularization framework of \citep{pmlr-v247-erez24a}. This approach yields regret $O(\sqrt{T\log n} + n\log(nT))$, which by Theorem~\ref{thm:agnostic-lower-bound} is optimal for proper learners up to logarithmic factors.

We conclude this section by showing how the estimator properties imply Theorem~\ref{thm:improper-upper-bound}; the full verification of these properties is
deferred to Appendix~\ref{sec:appendix-B-upper-bound-proof}. The full procedure is given in Algorithm~\ref{alg:improper}.

\paragraph{Proof (sketch).}
It suffices to analyze regret with respect to the shifted losses $d_t$, since the shift
$\Ind\{y_t=-1\}$ is the same for every hypothesis and therefore cancels in regret
comparisons. Let
\[
    \bar L_T := \sum_{t=1}^T \langle p_t,d_t\rangle,
    \qquad
    L_T(i) := \sum_{t=1}^T d_t(i),
\]
and define $\widehat{\bar L}_T$ and $\widehat L_T(i)$ analogously using
$\widehat d_t$. Let $i^* = \arg\min_i L_T(i)$. Since $d_t(h^+)=0$ and the algorithm samples from $p_t$ with probability $1-\rho$,
the learner's expected shifted loss is $(1-\rho)\E[\bar L_T]$. Thus
\[
    \E[\mathfrak R_T]
    \le
    \E[\bar L_T-L_T(i^*)]+\rho T,
\]
where the $\rho T$ term accounts for the exploration rounds and uses
$|d_t(i)|\le 1$. We now compare the true shifted losses to the estimated losses. Decompose
\[
    \bar L_T-L_T(i^*)
    =
    \underbrace{\widehat{\bar L}_T-\widehat L_T(i^*)}_{\text{Hedge}}
    +
    \underbrace{\bar L_T-\widehat{\bar L}_T}_{\text{bias}}
    +
    \underbrace{\widehat L_T(i^*)-L_T(i^*)}_{\text{deviation}} .
\]
The three estimator properties established above control these terms directly.
Boundedness allows the standard Hedge argument, and the constant
second-moment bound gives
\[
    \E[\widehat{\bar L}_T-\widehat L_T(i^*)]
    \le
    \frac{\log n}{\eta}+O(\eta T).
\]
The bias guarantee gives
$
    \E[\bar L_T-\widehat{\bar L}_T]\le O(\eta T),
$
while the uniform deviation guarantee gives
\[
    \E[\widehat L_T(i^*)-L_T(i^*)]
    \le
    \E\!\left[\max_i(\widehat L_T(i)-L_T(i))\right]
    \le
    \frac{\log n}{\eta}.
\]
Combining,
\[
    \E[\mathfrak R_T]
    \le
    O\!\left(\frac{\log n}{\eta}+\eta T+\rho T\right).
\]
With $
    \eta=\frac12\sqrt{\frac{\log n}{T}},
    \rho=\sqrt{\frac{\log n}{T}},$
this becomes $O(\sqrt{T\log n})$. If $T<\log n$, the trivial regret bound
$\mathfrak R_T\le T$ is already at most $\sqrt{T\log n}$, completing the proof.

\subsubsection{Lower Bound} In this section, we show a lower bound for proper learners. The lower bound is stated formally in the following Theorem.

\begin{theorem}\label{thm:agnostic-lower-bound}
For any (possibly randomized) proper learning algorithm and all integers
$n$, there exist a manipulation graph $G$ and a hypothesis class
$\mc H$ with $|\mc H| = n$ such that the algorithm must incur expected regret at least $\Omega\!\left(\min\!\left\{\sqrt{T \log n} + n, \sqrt{T n \log n}\right\}\right).$
\end{theorem}

\begin{proof}  We combine two hard instances. Let $(G_1,\mc H_1)$ be the realizable construction from
Section~\ref{sec:realizable-LB} with $\mc H_1=\{h_{1,1},\dots,h_{1,n}\}$, which yields regret $\Omega(\min\{n,T\})$ for any proper learner in the realizable setting which is a special case of the agnostic setting. Let $(G_2,\mc H_2)$ be an instance with no effective manipulation (e.g., $G_2$ has no edges), together with a class $\mc H_2=\{h_{2,1},\dots,h_{2,n}\}$ of Littlestone dimension $\Theta(\log n)$; by the standard agnostic online lower bound of \cite{ben2009agnostic}, any learner then incurs regret $\Omega(\sqrt{T\log n})$.

Define $G$ as the disjoint union of $G_1$ and $G_2$. Define $\mc H=\{h^1,\dots,h^n\}$ by setting, for each $i\in[n]$, $h^i|_{V(G_1)}=h_{1,i}$ and $h^i|_{V(G_2)}=h_{2,i}$, that is extend each hypothesis by matching
the $i$th hypothesis on the other component. At the start of the interaction, the adversary chooses which ever construction procedures larger regret (depending on $n$ and $T$). Since the constructions are effectively disconnected, the interaction is reduced to the corresponding original instance. Therefore every proper learner satisfies $ \E[\mathfrak R_T] \;=\; \Omega\!\left(\max\!\left\{\sqrt{T\log n},\min\{n,T\}\right\}\right).$ Finally, a routine case analysis (given in Appendix~\ref{prf:agnostic-lower-bound-case-analysis}) shows that
\[
\Omega\!\left(\max\!\left\{\sqrt{T \log n},\ \min\{n,T\}\right\}\right)
=
\Omega\!\left(\min\!\left\{\sqrt{T \log n} + n,\ \sqrt{T n \log n}\right\}\right),
\]
which completes the proof.
\end{proof}

\bibliography{refs}

\newpage
\appendix

\section{Proofs of Section \ref{sec:realizable}}\label{sec:appendix-A}

\subsection{Proof of Theorem \ref{thm:finite-realizable-upper-bound}}
Fix round $t$ and agent $(x_t,y_t)$ and recall the definition of $\delta_t$. We can equivalently view sampling $h_t$ as first flipping a coin that lands on heads with probability $p$. If the coin lands on heads then the learner sets $h_t = h^+$. Otherwise, if the coin lands on tails, $h_t \sim \mathrm{Unif}(V_t)$. Then, if the learner makes a mistake on round $t$ and flipped heads, we call this a \textit{exploration mistake}. Otherwise, we call it an \textit{expert mistake}. We can write the total number of expected mistakes as
\[\E[M(T)] = \E[\#\text{ of expert mistakes}] + \E[\#\text{ of exploration mistakes}].\]
Then 
\begin{equation}\label{eqn:head}\E[\#\text{ of exploration mistakes}] \leq \E[\#\text{ of heads}] = pT.\end{equation}
Next, we analyze $\E[\#\text{ of expert mistakes}]$. We first show that 
\begin{equation}\label{eqn:progress}\E[\log |V_t|/|V_{t+1}|] \geq p\delta_t.\end{equation}
Observe, that with probability $1-p$ we flip tails and set $|V_{t+1}| = |V_t|$ so $\log (|V_t|/|V_{t+1}|) = 0$. On the other hand, with probability $p$ we flip heads and have $|V_{t+1}| = (1-\delta_t)|V_t|$. Then
\[\log (|V_t|/|V_{t+1}|) = \log\left(\frac{1}{1-\delta_t}\right) \geq \delta_t\]
and the result follows. Note the last inequality follows from the fact that $\log(1/(1-x)) \geq x$ for $x \in (0,1]$ and base $e$. Next, let $A_t$ be event that a tail mistake happens on round $t$. Then 
\[\Pr[A_t] = \Pr[\text{tails}]\Pr[\text{learner makes a mistake} \mid \text{tails}] = (1-p)\delta_t.\]
Combining the above with \ref{eqn:progress} we have
\begin{equation}\label{eqn:tails}
    \Pr[A_t] \leq \left(\frac{1-p}{p}\right) \E[\log |V_t|/|V_{t+1}|]
\end{equation}
Then summing over the rounds and applying \ref{eqn:tails} we have
\begin{equation}\label{eqn:tails-total}
    \E[\#\text{ of expert mistakes}] = \sum_{t=1}^T \Pr[A_t] \leq \left(\frac{1-p}{p}\right)\E\left[\sum_{t=1}^T \log |V_t|/|V_{t+1}|\right] \leq \left(\frac{1-p}{p}\right) \log n. 
\end{equation}
Finally, combining \ref{eqn:head} and \ref{eqn:tails-total} we get
\[\E[M_A(T)] \leq pT + \left(\frac{1-p}{p}\right)\log n = O\left(\sqrt{T\log n}\right)\]
where the last equality follows by our choice of $p = \min(1, \sqrt{\log n/T})$.

\subsection{Pseudocode for Infinite Realizable Algorithm}\label{sec:pseudocode-for-infinite}

For an expert $E$, let $h_E$ denote the classifier currently output by its
copy of $\mathcal A$, and let $E(x,y)$ denote the expert obtained by feeding
$(x,y)$ to that copy.

\begin{algorithm2e}[t]
\caption{Expert-Mix}
\label{alg:mix-expert}
\DontPrintSemicolon

Set $p \leftarrow \min\!\left\{1,\sqrt{M_{\mathcal A}\log(2\bar\Delta)/T}\right\}$,
$\mathcal E_1 \leftarrow \{E_\emptyset\}$, and $w_1(E_\emptyset)\leftarrow 1$.\;

\For{$t \leftarrow 1$ \KwTo $T$}{
  Set $W_t \leftarrow \sum_{E\in\mathcal E_t} w_t(E)$ and sample
  $C_t \sim \mathrm{Ber}(p)$.\;

  \eIf{$C_t = 1$}{
    Play $h_t = h^+$ and receive $(z_t,y_t)=(x_t,y_t)$.\;
    Initialize $\mathcal E_{t+1}\leftarrow \emptyset$.\;

    \For{$E\in\mathcal E_t$}{
      \eIf{$\ell_G^{\mathrm{str}}(h_E,(x_t,y_t))=0$}{
        Add $E$ to $\mathcal E_{t+1}$ with weight $w_{t+1}(E)\leftarrow w_t(E)$.\;
      }{
        \eIf{$y_t=-1$}{
          Choose any $x\in N_G[x_t]$ such that $h_E(x)=+1$.\;
          Add $E' := E(x,-1)$ to $\mathcal E_{t+1}$ with weight
          $w_{t+1}(E')\leftarrow w_t(E)/2$.\;
        }{
          \For{$x\in N_G[x_t]$}{
            Add $E' := E(x,+1)$ to $\mathcal E_{t+1}$ with weight
            $w_{t+1}(E')\leftarrow w_t(E)/(2|N_G[x_t]|)$.\;
          }
        }
      }
    }
  }{
    Sample $E_t\in\mathcal E_t$ with probability $w_t(E_t)/W_t$.\;
    Play $h_t=h_{E_t}$ and receive $(z_t,y_t)$.\;
    Set $\mathcal E_{t+1}\leftarrow \mathcal E_t$ with unchanged weights.\;
  }
}
\end{algorithm2e}





\subsection{Proof of Theorem \ref{thm:infinite-realizable-upper-bound}}

\begin{lemma}\label{lem:exp-mistake-prob-inf}
    For all rounds $t$, we have
    \[
        \Pr[\text{expert mistake at round } t] \leq \frac{2(1-p)}{p} \cdot \E\left[\log\frac{W_t}{W_{t+1}}\right].
    \]
\end{lemma}
\begin{proof}
    For each expert $E$, denote by $z_E(x_t)$ the best response from $x_t$ of $h_E$. Let $A_t$ be the event that an \emph{expert} mistake occurs at round $t$. Then 
    \[
        \Pr[A_t] = (1-p) \delta_t, \tag{1}
    \]
    where $\delta_t$ is the weighted proportion of all experts which are incorrect at round $t$. That is
    \[
        \delta_t := \frac{\sum_{\{E \in \mathcal E_t : h_E(z_E(x_t)) \neq y_t\}} w_t(E) }{W_t}.
    \]
    Now, we can show that
    \[
        \E\left[\log\frac{W_t}{W_{t+1}}\right]
        \geq p\cdot \frac{\delta_t}{2} \tag{2}
    \]
    because on an exploration round (which happens with probability $p$), all incorrect experts are replaced by experts whose total weight is half that of those which were initially incorrect. This means that
    \[
        W_{t+1} \leq W_t (1-\delta_t/2).
    \]
    Using that $\log(1/(1-x)) \geq x$ for $x \in [0,1)$, we obtain
    \[
        \log \frac{W_t}{W_{t+1}} \geq \frac{\delta_t}{2}.
    \]
    We immediately deduce $(2)$ which together with $(1)$ finishes the lemma.
\end{proof}

\begin{lemma}\label{lem:realizable-expert-exists}
    At any round $t$, there is some expert $E \in \mathcal E_t$ whose associated sequence of agents is realizable and forces mistakes on each round.
\end{lemma}
\begin{proof}
    First, the true agents that the chooses are realizable throughout by a fixed $h^*$. We will in particular show there is $E_t$ whose sequence realizes $h^*$.
    We proceed by induction. The base case is trivial since the associated sequence of agents is empty. Let $E_t$ be an expert whose associated sequence of agents is realizable by $h^*$. If $E_t$ remains in $\mathcal E_{t+1}$ then there is nothing to prove as its associated sequence remains the same. Thus assume otherwise. This means that round $t+1$ is an exploration round with agent $(x_{t+1},y_{t+1})$ which $E_t$ is incorrect on. If $y_{t+1} = +1$, then realizability of $(x_{t+1},+1)$ under $h^*$ means that there is some $x^* \in N_G[x_{t+1}]$ such that $h^*(x^*) = +1$. We thus choose the expert $E_{t+1} := E_t(x^*,+1) \in \mathcal E_t$. This forces a mistake for $E_t$ because $E_t$ being incorrect on that round means it labeled the entire closed neighborhood of $x_{t+1}$ negative. If $y_{t+1} = -1$, then realizability forces $h^*$ to label all $N_G[x_{t+1}]$ all negative. We can then choose $E_{t+1} := E_t(x,-1) \in \mathcal E_{t+1}$ where $x \in N_G[x_{t+1}]$, and to enforce a mistake, we additionally ensure that $E_t$ labels $x$ as positive. This is possible because $E$ being incorrect on $(x_{t+1},-1)$ ensures such an $x$ exists.
\end{proof}


\subsection{Proof of Theorem \ref{thm:realizable-LB}} Recall the basic construction described in Section \ref{sec:realizable-LB} and the following notation for round $t$:
\begin{itemize}
    \item $\mc D_t$: learners distribution.
    \item $(x_t,y_t)$: agent chosen by the adversary.
    \item $f_t \sim \mc D_t$: learners sampled hypothesis.
    \item $z_t = \mathrm{BR}_{G,f_t}(x_t)$.
    \item $k(f_t) = \min\{j \in [n] : f_t(p_j) = +1\}$.
    \item $M_t := \Ind\{f_t(z_t)\neq y_t\}$.
    \item $\gamma = \Theta(1/\sqrt{T})$.
\end{itemize}

Define $M(T) = \sum_{t=1}^T M_t$. We say that round $t$ is a \emph{case 1--3 round} if $\mathcal D_t$ triggers one of cases 1--3, and a \emph{case 4 round} otherwise. 

Define
\[
\tau := \min\left\{\left\lfloor c_1 n\right\rfloor,\ \left\lfloor \frac{c_2}{\gamma}\right\rfloor,\ \left\lfloor \frac{T}{2}\right\rfloor\right\},
\]
where $c_1,c_2\in(0,1)$ are absolute constants to be chosen later. Let $t_1<t_2<\cdots$ denote the (random) times at which the adversary plays case~4, i.e.\ $t_r$ is the time index of the $r$-th case~4 round (when it exists). For $r\geq 0$, define
\[
\mathcal{E}_r
:= \bigcap_{q=1}^r \left\{ f_{t_q}\in\mathcal{G}\ \text{ and }\ z_{t_q}=p_{k(f_{t_q})}\right\}.
\]
Then $\mc E_r$ represents the event that the learner makes a mistake in each of the first $r$ case 4 rounds.

Theorem \ref{thm:realizable-LB} follows mainly from the following Lemma which provides a lower bound for the probability of $\mc E_\tau$.

\begin{lemma}\label{lem:prob-Etau}
For every $r\in[\tau]$,
\[
\Pr[\mathcal{E}_r \mid \mathcal{E}_{r-1}]
\;\geq\;
(1-3\gamma)\left(1-\frac{1}{n-r+1}\right).
\]
Consequently,
\[
\Pr[\mathcal{E}_\tau]
\;\geq\;
(1-3\gamma)^\tau \cdot \frac{n-\tau}{n}.
\]
\end{lemma}

We additionally rely on the following Lemma which characterizes the loss structure of the Learner in each round $t$.

\begin{lemma}\label{lem:loss-structure}
    Fix a round $t$ and the learner's hypothesis $f_t$. Then:
    \begin{enumerate}
        \item If the adversary uses case $1,2$, or $3$, then $\Pr[M_t = 1] \geq \gamma.$
        \item Otherwise, if the adversary uses case $4$ and $f_t \in \mc G$ with $z_t = p_{k(f_t)}$, then $M_t = 1$.
    \end{enumerate}
\end{lemma}

Assuming the above two Lemmas, we now complete the proof of Theorem \ref{thm:realizable-LB}.

\begin{proof}[Proof of Theorem \ref{thm:realizable-LB}.] 
If there are fewer than $\tau$ case~4 rounds, then there are at least $T-\tau\geq T/2$ case~1--3 rounds, and by item 1 in Lemma~\ref{lem:loss-structure} the learner incurs at least $\gamma$ mistake probability
in each such round. Hence
\[
\E[M(T)] \;\geq\; \gamma\cdot \frac{T}{2}.
\]
Otherwise, the first $\tau$ case~4 rounds exist and by item 2 in Lemma \ref{lem:loss-structure} and Lemma~\ref{lem:prob-Etau},
\[
\E[M(T)]
\;\geq\;
\tau\cdot \Pr[\mathcal{E}_\tau]
\;\geq\;
\tau \cdot (1-3\gamma)^\tau \cdot \frac{n-\tau}{n}.
\]
Choose $c_1=\frac12$ and $c_2=\frac16$. Then $\tau\leq n/2$ implies $(n-\tau)/n \geq 1/2$, and $\tau\leq \frac{1}{6\gamma}$ implies
\[
(1-3\gamma)^\tau \geq 1-3\gamma\tau \geq \frac12.
\]
Therefore $\Pr[\mathcal{E}_\tau]\geq 1/4$ and hence
\[
\E[M(T)] \;\geq\; \frac{\tau}{4}.
\]

Finally, plugging in for $\gamma$ gives
$\tau=\Theta(\min\{n,\sqrt{T}\})$, and thus
\[
\E[M(T)] = \Omega\left(\min\left\{\sqrt{T},\,n\right\}\right),
\]
as claimed.
\end{proof}

In the remainder of the section, we prove the two Lemmas we assumed. First we prove Lemma \ref{lem:loss-structure}. 

\begin{proof}[Proof of Lemma \ref{lem:loss-structure}.]

We first show item 1. This is done in three cases.

\smallskip\noindent\textbf{Case 1:} $\mathcal D_t(A)\geq\gamma$, adversary plays $(L,-1)$.
If $f_t\in A$, then either $f_t(L)=+1$ (agent stays at $L$ and is labeled $+1$) or some $u_k$ has $f_t(u_k)=+1$ (agent moves to such a $u_k$ and is labeled $+1$). In either event the learner predicts $+1$ while $y_t=-1$, hence $M_t=1$. Thus
\[
  \Pr[M_{t}=1] \geq \mathcal D_t(A)\geq\gamma.
\]

\smallskip\noindent\textbf{Case 2:} $\mathcal D_t(B)\geq\gamma$, adversary plays $(z,-1)$.
If $f_t\in B$, then either $f_t(z)=+1$ (agent stays at $z$) or $f_t(R)=+1$ (agent moves to $R$), in either case predicting $+1$ on label $-1$. Hence $M_t=1$, and
\[
  \Pr[M_{t}=1] \geq \mathcal D_t(B) \geq \gamma.
\]

\smallskip\noindent\textbf{Case 3:} $\mathcal D_t(C)\geq \gamma$, adversary plays $(R,+1)$. If $f_t \in C$, then $f_t(p_j)=-1$ for all $j$ and $f_t(R)=f_t(z)=-1$, so the agent at $R$ has no positive reachable in one step and stays at $R$. The learner predicts $-1$ on label $+1$, so $M_t=1$. Therefore
\[
  \Pr[M_{t}=1] \geq \mc D_t(C) \geq \gamma.
\]

\smallskip\noindent Now we prove item 2. In case~4 we have $(x_t,y_t)=(u_{i^*},-1)$. If $z_t=p_{k(f_t)}$ and $f_t\in\mathcal{G}$,
then by definition we have $f_t(p_{k(f_t)})=+1$. Thus the learner predicts $f_t(z_t)=+1$ while $y_t=-1$, so $M_t=1$.
\end{proof}

To prove Lemma \ref{lem:prob-Etau}, we first introduce the notion of a \emph{transcript} and prove several important properties. 

Formally, let
\[
\mathcal{T}_t := (\mathcal D_s, f_s, z_s, y_s)_{s=1}^{t-1}
\]
denote the learner's transcript at the start of round $t$ (i.e., everything the learner has observed and recorded up to time $t-1$). We will track which indices are \emph{eliminated} by non-revealing case 4 rounds. Define
\begin{align*}
E_t(\tau)
&:= \{\, k(f_s) \in[n] : s<t,\ \text{round $s$ is case 4, } f_s\in\mathcal{G},\ \text{and } z_s = p_{k(f_s)} \,\},\\
S_t(\tau)
&:= [n]\setminus E_t(\tau).
\end{align*}
We say that a transcript $\tau$ is a \emph{failure transcript up to time $t$} if every past case 4 round was both:
(i) sampled from $\mathcal{G}$, and (ii) \emph{non-revealing}, meaning the observed move was $z_s=p_{k(f_s)}$.
Formally, define 
\[
\mathrm{Fail}_t
:=
\left\{\mathcal{T}_t:\ \forall s<t\text{ with round $s$ case 4, we have } f_s\in\mathcal{G}\text{ and } z_s=p_{k(f_s)}\right\}.
\]

The following Lemma shows that if $i^*$ has not been eliminated at round $t$ then the posterior of $i^*$ is uniform. 

\begin{lemma}\label{lem:posterior-uniform}
Fix $t$ and let $\tau$ be any realization with $\Pr[\mathcal{T}_t=\tau]>0$ and $\tau\in\mathrm{Fail}_t$.
Then
\[
\Pr[i^*=j \mid \mathcal{T}_t=\tau]
\;=\;
\frac{\Ind\{j\in S_t(\tau)\}}{|S_t(\tau)|}.
\]
In particular, for every $j\in S_t(\tau)$,
\[
\Pr[i^*=j \mid \mathcal{T}_t=\tau]
\;=\;
\frac{1}{|S_t(\tau)|}.
\]
\end{lemma}

\begin{proof}
Fix $\tau$ with $\Pr[\mathcal{T}_t=\tau]>0$ and $\tau\in\mathrm{Fail}_t$.
Write $E:=E_t(\tau)$ and $S:=S_t(\tau)=[n]\setminus E$. For any $j \in E$, $\Pr[i^* = j \mid \mc T_t = \tau] = 0$. Thus, it suffices to only consider $j \in S$. We finish the proof in two steps.
\smallskip

\noindent\emph{Step 1:} We show that $\Pr[\mathcal{T}_t=\tau\mid i^*=j]$ is the same for all $j\in S$. Fix any $j\in S$ and consider generating the transcript round-by-round.
For each $s<t$, conditional on the prefix transcript $(\mathcal D_r,f_r,z_r,y_r)_{r=1}^{s-1}$,
the learner's choice of $\mathcal D_s$ and its sample $f_s\sim\mathcal D_s$ depend only on the learner's internal randomness and the prefix, and therefore are conditionally independent of $i^*$.
Moreover, the adversary's case choice at time $s$ is a deterministic function of $\mathcal D_s$ (hence independent of $i^*$ given the prefix).

\begin{itemize}
\item If round $s$ is case 1--3, then $(x_s,y_s)\in\{(L,-1),(z,-1),(R,+1)\}$ does not depend on $i^*$. Hence $z_s=\mathsf{BR}_{f_s}(x_s)$ and $y_s$ are determined by $(f_s,\mathcal D_s)$ and do not depend on $i^*$.

\item If round $s$ is case 4, then $\tau\in\mathrm{Fail}_t$ implies $f_s\in\mathcal{G}$ and $z_s=p_{k(f_s)}$ in $\tau$. Let $k_s:=k(f_s)$. Since $k_s\in E$ (by definition of $E$) we have $j\neq k_s$ for every $j\in S$. For such $j\neq k_s$, the best-response rule for $u_j$ under $f_s\in\mathcal{G}$ gives
\[
\mathrm{BR}_{G,f_s}(u_j)=p_{k_s},
\]
so the observed move matches $\tau$ for every $i^*=j\in S$.
\end{itemize}

Therefore, at every round $s<t$, the conditional probability of producing the $s$-th transcript element of $\tau$ given the prefix is identical for all $i^*=j$ with $j\in S$. Thus, we conclude that
\[
\Pr[\mathcal{T}_t=\tau \mid i^*=j]=\lambda(\tau)\qquad\text{for all } j\in S
\]
for some value $\lambda(\tau)\geq 0$.

\smallskip

\noindent\emph{Step 2:}
Using $\Pr[i^*=j]=1/n$ and Step 1,
\[
\Pr[i^*=j\mid \mathcal{T}_t=\tau]
=
\frac{\Pr[\mathcal{T}_t=\tau\mid i^*=j]\Pr[i^*=j]}{\sum_{j'\in[n]}\Pr[\mathcal{T}_t=\tau\mid i^*=j']\Pr[i^*=j']}
=
\frac{\lambda(\tau)\cdot (1/n)}{\sum_{j'\in S}\lambda(\tau)\cdot (1/n)}
=
\frac{1}{|S|}
\]
for $j\in S$, and it is $0$ for $j\in E$.
\end{proof}

Next make a simple observation.

\begin{lemma}\label{lem:case4-mass-on-G}
If round $t$ is case~4, then $\mathcal{D}_t(\mathcal{G}) \geq 1-3\gamma$.
Equivalently, $\Pr[f_t\in\mathcal{G} \mid \mathcal{T}_t] \geq 1-3\gamma$.
\end{lemma}

\begin{proof}
If round $t$ is case~4, then by definition of the adversary we have
$\mathcal{D}_t(A)<\gamma$, $\mathcal{D}_t(B)<\gamma$, and $\mathcal{D}_t(C)<\gamma$. Since $\mathcal{G}=(A\cup B\cup C)^c$,
\[
\mathcal{D}_t(\mathcal{G})
=1-\mathcal{D}_t(A\cup B\cup C)
\geq 1-\mathcal{D}_t(A)-\mathcal{D}_t(B)-\mathcal{D}_t(C)
>1-3\gamma.
\]
Conditioning on $\mathcal{T}_t$, the draw $f_t\sim\mathcal{D}_t$ implies
$\Pr[f_t\in\mathcal{G}\mid \mathcal{T}_t]=\mathcal{D}_t(\mathcal{G})$.
\end{proof}

We are now ready to prove Lemma \ref{lem:prob-Etau}.

\begin{proof}[Proof of Lemma \ref{lem:prob-Etau}.]
Fix $r\in[\tau]$ and condition on $\mathcal{E}_{r-1}$.
On this event, the transcript at the start of round $t_r$ is a failure transcript in the sense
of $\mathrm{Fail}_{t_r}$. Let $\mathcal{T}_{t_r}$ denote the transcript at the start of round $t_r$.
By Lemma~\ref{lem:case4-mass-on-G},
\[
\Pr[f_{t_r}\in\mathcal{G}\mid \mathcal{T}_{t_r}] \geq 1-3\gamma.
\]

Next assume $f_{t_r}\in\mathcal{G}$ and write $k_r := k(f_{t_r})$.
In Case~4 the underlying vertex is $u_{i^*}$. If $i^*\neq k_r$, then by the definition of $\mathsf{BR}_{G,f_{t_r}}(\cdot)$ we have
$\mathrm{BR}_{G,f_{t_r}}(u_{i^*})=p_{k_r}$, i.e.\ $z_{t_r}=p_{k_r}$.
Thus
\[
\Pr[z_{t_r} = p_{k_r}\mid \mathcal{T}_{t_r}, f_{t_r}]
\;\geq\;
\Pr[i^* \neq k_r \mid \mathcal{T}_{t_r}]
=
1-\Pr[i^*=k_r\mid \mathcal{T}_{t_r}].
\]
Because we are conditioning on a failure transcript, Lemma~\ref{lem:posterior-uniform} applies:
\[
\Pr[i^*=k_r\mid \mathcal{T}_{t_r}]
\leq \frac{1}{|S_{t_r}(\mathcal{T}_{t_r})|}.
\]
Moreover, on $\mathcal{E}_{r-1}$ the eliminated set has size at most $r-1$,
so $|S_{t_r}(\mathcal{T}_{t_r})|\geq n-(r-1)=n-r+1$.
Therefore,
\[
\Pr[z_{t_r}=p_{k_r} \mid \mathcal{T}_{t_r}, f_{t_r}]
\;\geq\;
1-\frac{1}{n-r+1}.
\]

Combining the above two equations and using that the above bound holds for every realization of
$\mathcal{T}_{t_r}$ under $\mathcal{E}_{r-1}$, we get
\[
\Pr[\mathcal{E}_r \mid \mathcal{E}_{r-1}]
=
\Pr\left[f_{t_r}\in\mathcal{G}\ \text{and}\ z_{t_r}=p_{k(f_{t_r})}\mid \mathcal{E}_{r-1}\right]
\;\geq\;
(1-3\gamma)\left(1-\frac{1}{n-r+1}\right).
\]

Multiplying over $r=1,\dots,\tau$ yields
\[
\Pr[\mathcal{E}_\tau]
=
\prod_{r=1}^\tau \Pr[\mathcal{E}_r \mid \mathcal{E}_{r-1}]
\;\geq\;
(1-3\gamma)^\tau \cdot \prod_{r=1}^\tau \left(1-\frac{1}{n-r+1}\right).
\]
Finally, the product telescopes:
\[
\prod_{r=1}^\tau \left(1-\frac{1}{n-r+1}\right)
=
\prod_{m=n-\tau+1}^{n} \left(1-\frac{1}{m}\right)
=
\prod_{m=n-\tau+1}^{n} \frac{m-1}{m}
=
\frac{n-\tau}{n}.
\]
\end{proof}

\subsection{Proof of Corollary \ref{cor:realizable-lb-inf}}\label{sec:appendix-A-proof-of-infinite-lb}.
\begin{proof} Recall the construction in Theorem~\ref{thm:realizable-LB} with parameter $n$. That instance has $|\mc H|=n$, $\text{Ldim}(\mc H)=1$, and max degree $\Delta=\Theta(n)$, and forces $\Omega(\min\{n,\sqrt{T}\})$ expected mistakes for any randomized learner.

To obtain Littlestone dimension $d$, take $d$ disjoint copies of this instance and let $G$ be their disjoint union.
Define the hypothesis class $\mc H$ to be the product class that, on each copy, behaves like one of the $n$
singleton hypotheses (so $|\mc H|=n^d$). Since copies are independent, $\Delta(G)=\Theta(n)$ and
$\text{Ldim}(\mc H)=d$.

Now run the lower-bound adversary independently in each copy for $T/d$ rounds. Since these copies are disjoint the Theorem~\ref{thm:realizable-LB} lower bound applies in each copy, yielding expected mistakes
\[
\Omega\!\left(d\cdot \min\left\{n,\sqrt{T/d}\right\}\right)
=
\Omega\!\left(\min\left\{d\Delta,\sqrt{Td}\right\}\right).
\]
\end{proof}

\section{Proofs of Section \ref{sec:agnostic}}\label{sec:appendix-B} 

\subsection{Proof of Theorem~\ref{thm:improper-upper-bound}}
\label{sec:appendix-B-upper-bound-proof}

Throughout the proof, write $\mc H = \{h^1, \ldots, h^n\}$. We can assume that $T \geq \log n$, as the bound is trivial otherwise. Let
\[
    \eta = \frac12\sqrt{\frac{\log n}{T}}, \qquad \rho = \sqrt{\frac{\log n}{T}}.
\]
Then $\rho = 2\eta$, and since $T \geq \log n$, we have $\rho \leq 1$. Recall that for each round $t$,
\[
    R_t := \{h^i \in \mc H : h^i(x') = -1 \text{ for all } x' \in N_G[x_t]\},
\]
and the shifted loss is
\[
    d_t(i) := \ell_G^{\mathrm{str}}(h^i, (x_t, y_t)) - \Ind\{y_t = -1\} = y_t \Ind\{h^i \in R_t\}.
\]
We extend this definition to the improper classifier $h^+$ by setting $d_t(h^+) := \ell_G^{\mathrm{str}}(h^+, (x_t, y_t)) - \Ind\{y_t = -1\}$. Since $h^+$ always predicts $+1$, the strategic loss is exactly $\Ind\{y_t = -1\}$, and therefore $d_t(h^+) = 0$ for every $t$. Moreover, for every $i \in [n]$,
\[
    \ell_G^{\mathrm{str}}(h^i, (x_t, y_t)) = d_t(i) + \Ind\{y_t = -1\},
\]
so the additive shift cancels in regret comparisons:
\[
    \sum_{t=1}^T \ell_G^{\mathrm{str}}(h_t, (x_t, y_t)) - \min_{i \in [n]} \sum_{t=1}^T \ell_G^{\mathrm{str}}(h^i, (x_t, y_t)) = \sum_{t=1}^T d_t(h_t) - \min_{i \in [n]} \sum_{t=1}^T d_t(i).
\]
It therefore suffices to bound the regret in shifted-loss form. Let $\mc F_{t-1}$ denote the history before round $t$, and define $\mc G_t := \sigma(\mc F_{t-1}, x_t, y_t)$. Conditioned on $\mc G_t$, the example $(x_t, y_t)$, the distribution $p_t$, and the set $R_t$ are all fixed; the only remaining randomness is the learner's. Recall the reveal event
\[
    A_t := \Ind\{h_t = h^+ \text{ or } h_t(z_t) = -1\},
\]
under which the learner certifies $z_t = x_t$ and reconstructs $R_t$. For all rounds, define
\[
    q_t := p_t(R_t) := \sum_{h^i \in R_t} p_t(i), \qquad a_t := \rho + (1-\rho)q_t,
\]
and recall the estimator
\begin{equation}
    \widehat d_t(i) = A_t \cdot \frac{y_t \Ind\{h^i \in R_t\}}{a_t + \eta y_t}.
    \label{eq:appendix-estimator}
\end{equation}

\paragraph{Overview of the proof.} The argument proceeds in four steps. We first compute the conditional probability of the reveal event (Lemma~\ref{lem:reveal-probability}). We then establish two elementary properties of the estimator: a constant second moment and a small downward bias (Lemma~\ref{lem:second-moment-bias}). The third step is an exponential moment bound (Lemma~\ref{lem:exponential-moment}) which yields uniform control of the estimation error over all hypotheses (Corollary~\ref{cor:uniform-deviation}). Finally, we combine these ingredients with a standard Hedge potential bound (Lemma~\ref{lem:hedge-potential}) to obtain the regret guarantee.

\paragraph{Step 1.} We begin by identifying the conditional probability that the learner observes a reveal.

\begin{lemma}\label{lem:reveal-probability}
For every round $t$, $\Pr(A_t = 1 \mid \mc G_t) = a_t$.
\end{lemma}

\begin{proof}
Condition on $\mc G_t$. If the learner plays $h^+$, the agent has no incentive to manipulate (since $h^+(x_t) = +1$), so $z_t = x_t$ and $A_t = 1$. If instead the learner plays $h^i \in \mc H$, then $A_t = 1$ exactly when $h^i(z_t) = -1$. We claim this is equivalent to $h^i \in R_t$. If $h^i \in R_t$, then $h^i$ labels every point in $N_G[x_t]$ as $-1$, so the agent has no profitable manipulation and $z_t = x_t$, giving $h^i(z_t) = -1$. Conversely, if $h^i \notin R_t$, then either $h^i(x_t) = +1$ (the agent does not move and receives prediction $+1$) or some neighbor of $x_t$ is labeled $+1$ (the agent moves there and receives $+1$); in both cases $h^i(z_t) = +1$. Combining, $\Pr(A_t = 1 \mid \mc G_t) = \rho + (1-\rho)p_t(R_t) = a_t$.
\end{proof}

\paragraph{Step 2.} The next two estimates control how $\widehat d_t$ relates to $d_t$ on average. The second-moment bound \eqref{eq:appendix-second-moment} controls the variance term in the Hedge analysis, while the bias bound \eqref{eq:appendix-bias} controls the gap between the cumulative loss $\bar L_T$ and its estimate $\widehat{\bar L}_T$.

\begin{lemma}\label{lem:second-moment-bias}
For every round $t$,
\begin{equation}
    \E\!\left[\inner{p_t}{\widehat d_t^{\,2}} \mid \mc G_t\right] \leq 4,
    \label{eq:appendix-second-moment}
\end{equation}
and
\begin{equation}
    \E\!\left[\inner{p_t}{d_t - \widehat d_t} \mid \mc G_t\right] \leq 2\eta.
    \label{eq:appendix-bias}
\end{equation}
\end{lemma}

\begin{proof}
Let $q_t = p_t(R_t)$ and $a = a_t$. By Lemma~\ref{lem:reveal-probability}, $A_t \mid \mc G_t \sim \mathrm{Bernoulli}(a)$. For the second moment, using \eqref{eq:appendix-estimator},
\[
    \E\!\left[\inner{p_t}{\widehat d_t^{\,2}} \mid \mc G_t\right]
    = \sum_{h^i \in R_t} p_t(i) \frac{a}{(a+\eta y_t)^2}
    = \frac{a q_t}{(a+\eta y_t)^2}.
\]
If $y_t = +1$, then $a q_t / (a+\eta)^2 \leq q_t/a \leq 1$, where the last inequality uses $a = \rho + (1-\rho)q_t \geq q_t$. If $y_t = -1$, then $a \geq \rho = 2\eta$, so $a - \eta \geq a/2$ and hence $a q_t / (a-\eta)^2 \leq 4 q_t/a \leq 4$. This proves \eqref{eq:appendix-second-moment}. For the bias bound, $h^i \notin R_t$ contributes zero, while for $h^i \in R_t$,
\[
    d_t(i) = y_t, \qquad \E[\widehat d_t(i) \mid \mc G_t] = \frac{a y_t}{a + \eta y_t}.
\]
Therefore
\[
    \E\!\left[\inner{p_t}{d_t - \widehat d_t} \mid \mc G_t\right]
    = \sum_{h^i \in R_t} p_t(i) \left(y_t - \frac{a y_t}{a + \eta y_t}\right).
\]
If $y_t = +1$, this is $q_t \cdot \eta/(a+\eta) \leq \eta$. If $y_t = -1$, this is $\eta q_t / (a-\eta) \leq 2\eta q_t/a \leq 2\eta$, where we used $a-\eta \geq a/2$ and $q_t \leq a$.
\end{proof}

\paragraph{Step 3.} The bias bound from Step 2 controls $\bar L_T - \widehat{\bar L}_T$ on average. To compare $\widehat L_T(i^*)$ to $L_T(i^*)$ for the random comparator $i^*$, however, we need a stronger guarantee: a uniform bound on $\max_i (\widehat L_T(i) - L_T(i))$. Towards that end, we first show an exponential moment bound.

\begin{lemma}\label{lem:exponential-moment}
For every fixed $i \in [n]$ and every round $t$,
\begin{equation}
    \E\!\left[\exp\left(\eta(\widehat d_t(i) - d_t(i))\right) \mid \mc G_t\right] \leq 1.
    \label{eq:appendix-one-step-exp}
\end{equation}
Consequently, defining $M_t(i) := \exp(\eta \sum_{s=1}^t (\widehat d_s(i) - d_s(i)))$,
\begin{equation}
    \E[M_T(i)] \leq 1 \qquad \text{for every } i \in [n].
    \label{eq:appendix-exp-moment}
\end{equation}
\end{lemma}

\begin{proof}
Fix $i \in [n]$ and condition on $\mc G_t$. If $h^i \notin R_t$, both $d_t(i)$ and $\widehat d_t(i)$ vanish, and \eqref{eq:appendix-one-step-exp} is immediate. Suppose $h^i \in R_t$, and write $a = a_t$.

\smallskip\noindent\emph{Case $y_t = +1$.} Here $d_t(i) = 1$ and $\widehat d_t(i) = A_t/(a+\eta)$, so
\[
    \E\!\left[e^{\eta(\widehat d_t(i) - d_t(i))} \mid \mc G_t\right]
    = e^{-\eta}\left[1 - a + a \exp\left(\frac{\eta}{a+\eta}\right)\right].
\]
Setting $r = \eta/(a+\eta) \in [0, 1)$, the inequality $e^r \leq 1/(1-r)$ gives $\exp(\eta/(a+\eta)) \leq (a+\eta)/a$, and so
\[
    1 - a + a \exp\left(\frac{\eta}{a+\eta}\right) \leq 1 - a + (a+\eta) = 1 + \eta \leq e^\eta.
\]
Thus the conditional expectation is at most $e^{-\eta} \cdot e^\eta = 1$.

\smallskip\noindent\emph{Case $y_t = -1$.} Here $d_t(i) = -1$ and $\widehat d_t(i) = -A_t/(a-\eta)$, so
\[
    \E\!\left[e^{\eta(\widehat d_t(i) - d_t(i))} \mid \mc G_t\right]
    = e^{\eta}\left[1 - a + a \exp\left(-\frac{\eta}{a-\eta}\right)\right].
\]
Setting $r = \eta/a \in [0, 1/2]$ (since $a \geq 2\eta$), we have $\eta/(a-\eta) = r/(1-r)$. The inequality $-\log(1-r) \leq r/(1-r)$ implies $\exp(-r/(1-r)) \leq 1 - r$, equivalently $\exp(-\eta/(a-\eta)) \leq 1 - \eta/a$. Hence
\[
    1 - a + a \exp\left(-\frac{\eta}{a-\eta}\right) \leq 1 - a + a(1 - \eta/a) = 1 - \eta \leq e^{-\eta},
\]
and the conditional expectation is again at most $e^{\eta} \cdot e^{-\eta} = 1$. This proves \eqref{eq:appendix-one-step-exp}.

For \eqref{eq:appendix-exp-moment}, let $S_t(i) := \sum_{s=1}^t (\widehat d_s(i) - d_s(i))$ and note $M_t(i) = e^{\eta S_t(i)}$. Since $S_{t-1}(i)$ is $\mc F_{t-1}$-measurable and $\mc F_{t-1} \subseteq \mc G_t$,
\[
    \E[M_t(i) \mid \mc G_t] = M_{t-1}(i) \cdot \E\!\left[e^{\eta(\widehat d_t(i) - d_t(i))} \mid \mc G_t\right] \leq M_{t-1}(i).
\]
Taking conditional expectation over $\mc F_{t-1}$ via the tower rule shows $(M_t(i))_{t \geq 0}$ is a nonnegative supermartingale, so $\E[M_T(i)] \leq M_0(i) = 1$.
\end{proof}

The exponential moment bound \eqref{eq:appendix-exp-moment} immediately yields uniform control of the deviation via the log-sum-exp inequality.

\begin{corollary}\label{cor:uniform-deviation}
Let $L_T(i) := \sum_{t=1}^T d_t(i)$ and $\widehat L_T(i) := \sum_{t=1}^T \widehat d_t(i)$. Then
\begin{equation}
    \E\!\left[\max_{i \in [n]} \bigl(\widehat L_T(i) - L_T(i)\bigr)\right] \leq \frac{\log n}{\eta}.
    \label{eq:appendix-uniform-deviation}
\end{equation}
\end{corollary}

\begin{proof}
By log-sum-exp,
\[
    \max_{i \in [n]} \bigl(\widehat L_T(i) - L_T(i)\bigr) \leq \frac{1}{\eta} \log \sum_{i=1}^n \exp\left(\eta(\widehat L_T(i) - L_T(i))\right).
\]
Taking expectations and using Jensen's inequality and \eqref{eq:appendix-exp-moment},
\[
    \E\!\left[\max_{i \in [n]} \bigl(\widehat L_T(i) - L_T(i)\bigr)\right]
    \leq \frac{1}{\eta} \log \sum_{i=1}^n \E[M_T(i)]
    \leq \frac{\log n}{\eta}.
\]
\end{proof}

\paragraph{Step 4.}  The remaining ingredient is the standard Hedge guarantee, which compares the algorithm's loss on the estimators to that of any fixed expert. For completeness we give the argument below.

\begin{lemma}\label{lem:hedge-potential}
For every $j \in [n]$,
\begin{equation}
    \widehat{\bar L}_T - \widehat L_T(j) \leq \frac{\log n}{\eta} + \eta \sum_{t=1}^T \inner{p_t}{\widehat d_t^{\,2}},
    \label{eq:appendix-hedge}
\end{equation}
where $\widehat{\bar L}_T := \sum_{t=1}^T \inner{p_t}{\widehat d_t}$.
\end{lemma}

\begin{proof}
We first verify $-\eta \widehat d_t(i) \leq 1$, which is what allows us to use the inequality $e^x \leq 1 + x + x^2$. If $\widehat d_t(i) \geq 0$ this is immediate. If $\widehat d_t(i) < 0$, then necessarily $y_t = -1$, $A_t = 1$, and $h^i \in R_t$, so
\[
    -\eta \widehat d_t(i) = \frac{\eta}{a_t - \eta} \leq \frac{\eta}{\rho - \eta} = 1.
\]

Let $W_t = \sum_i w_t(i)$. Using $e^x \leq 1 + x + x^2$ for $x \leq 1$ and then $1+x \leq e^x$,
\[
    \frac{W_{t+1}}{W_t} = \sum_{i=1}^n p_t(i) e^{-\eta \widehat d_t(i)} \leq 1 - \eta \inner{p_t}{\widehat d_t} + \eta^2 \inner{p_t}{\widehat d_t^{\,2}},
\]
and so
\[
    \log \frac{W_{t+1}}{W_t} \leq -\eta \inner{p_t}{\widehat d_t} + \eta^2 \inner{p_t}{\widehat d_t^{\,2}}.
\]
Summing over $t$ and using $W_1 = n$ and $W_{T+1} \geq w_{T+1}(j) = \exp(-\eta \widehat L_T(j))$ for any fixed $j$,
\[
    -\eta \widehat L_T(j) - \log n \leq -\eta \widehat{\bar L}_T + \eta^2 \sum_{t=1}^T \inner{p_t}{\widehat d_t^{\,2}}.
\]
Rearranging gives \eqref{eq:appendix-hedge}.
\end{proof}

\subsubsection{Completing the Proof}

\begin{proof}[Proof of Theorem \ref{thm:improper-upper-bound}.]
Define the cumulative shifted losses
\[
    \bar L_T := \sum_{t=1}^T \inner{p_t}{d_t}, \qquad L_T(i) := \sum_{t=1}^T d_t(i),
\]
together with their estimated counterparts $\widehat{\bar L}_T$ and $\widehat L_T(i)$ as above. Let $i^* \in \arg\min_{i \in [n]} L_T(i)$, with deterministic tie-breaking. Since $d_t(h^+) = 0$ and the algorithm samples from $p_t$ with probability $1-\rho$,
\[
    \E[d_t(h_t) \mid \mc G_t] = (1-\rho) \inner{p_t}{d_t},
\]
so $\E[\sum_t d_t(h_t)] = (1-\rho) \E[\bar L_T]$. The expected shifted-loss regret therefore decomposes as
\[
    \E\!\left[\sum_{t=1}^T d_t(h_t) - \min_{i \in [n]} L_T(i)\right]
    = \E[\bar L_T - L_T(i^*)] - \rho \E[\bar L_T].
\]
Since $|d_t(i)| \leq 1$ for every $i, t$, we have $|\bar L_T| \leq T$, and so
\begin{equation}
    \E\!\left[\sum_{t=1}^T d_t(h_t) - \min_{i \in [n]} L_T(i)\right]
    \leq \E[\bar L_T - L_T(i^*)] + \rho T.
    \label{eq:appendix-exploration-cost}
\end{equation}
The remainder of the proof bounds the main term $\E[\bar L_T - L_T(i^*)]$. We decompose
\[
    \bar L_T - L_T(i^*) = \underbrace{\bigl(\widehat{\bar L}_T - \widehat L_T(i^*)\bigr)}_{\text{Hedge}} + \underbrace{\bigl(\bar L_T - \widehat{\bar L}_T\bigr)}_{\text{bias}} + \underbrace{\bigl(\widehat L_T(i^*) - L_T(i^*)\bigr)}_{\text{deviation}}
\]
and bound each term using the lemmas of the previous steps. Applying Lemma \ref{lem:hedge-potential} to $i^*$ and using \eqref{eq:appendix-second-moment}, we have
\[
    \E\!\left[\widehat{\bar L}_T - \widehat L_T(i^*)\right] \leq \frac{\log n}{\eta} + \eta \sum_{t=1}^T \E\!\left[\inner{p_t}{\widehat d_t^{\,2}}\right] \leq \frac{\log n}{\eta} + 4\eta T.
\]
For the bias term, \eqref{eq:appendix-bias} gives
\[
    \E\!\left[\bar L_T - \widehat{\bar L}_T\right] = \sum_{t=1}^T \E\!\left[\inner{p_t}{d_t - \widehat d_t}\right] \leq 2\eta T.
\]
Finally, for the deviation term, Corollary~\ref{cor:uniform-deviation} gives
\[
    \E\!\left[\widehat L_T(i^*) - L_T(i^*)\right] \leq \E\!\left[\max_{i \in [n]} \bigl(\widehat L_T(i) - L_T(i)\bigr)\right] \leq \frac{\log n}{\eta}.
\]
Summing the three bounds,
\begin{equation}
    \E[\bar L_T - L_T(i^*)] \leq \frac{2\log n}{\eta} + 6\eta T.
    \label{eq:appendix-main-term}
\end{equation}
Plugging \eqref{eq:appendix-main-term} into \eqref{eq:appendix-exploration-cost},
\[
    \E\!\left[\sum_{t=1}^T d_t(h_t) - \min_{i \in [n]} L_T(i)\right] \leq \frac{2\log n}{\eta} + 6\eta T + \rho T.
\]
With $\eta = \tfrac{1}{2}\sqrt{\log n / T}$ and $\rho = \sqrt{\log n / T}$,
\[
    \frac{2 \log n}{\eta} = 4\sqrt{T \log n}, \qquad 6\eta T = 3\sqrt{T \log n}, \qquad \rho T = \sqrt{T \log n},
\]
giving the bound $8\sqrt{T \log n}$. Since shifted-loss regret equals strategic-loss regret for every realization of the randomness, this completes the proof.
\end{proof}

\subsection{Proof of Theorem \ref{thm:agnostic-lower-bound}.}\label{prf:agnostic-lower-bound-case-analysis}
We prove that 
\[
\Omega\!\left(\max\!\left\{\sqrt{T \log n},\ \min\{n,T\}\right\}\right)
=
\Omega\!\left(\min\!\left\{\sqrt{T \log n} + n,\ \sqrt{T n \log n}\right\}\right).
\]
\begin{proof} Let $n\ge 2,\, T\ge 1,\, L:=\log n$, and $s:=\sqrt{TL}$. Define
\[
A=\min\{T,\max\{\min\{n,T\},s\}\}, \qquad 
B=\min\{T,\, n+s,\, \sqrt n\, s\}.
\]
We show that $\frac12 A\le B\le 2A$ which will imply the desired equality. \footnote{Note for every regret bound there is always an implicit $\min\{T, \cdot\}$.}

\paragraph{Case 1: $s\ge T$.} 
In this case, $\max\{\min\{n,T\}, s\} = s \ge T$, so $A = \min\{T, s\} = T$. 
For $B$, we have $n+s \ge T$ and $\sqrt{n}s \ge \sqrt{2}T > T$, so $B = T$. 
Thus $A=B$, and the inequality holds.

\paragraph{Case 2: $s<T$.} 
Here $\max\{\min\{n,T\}, s\} \le T$, so the outer minimum in $A$ is redundant. 
Thus $A = \max\{\min\{n,T\}, s\}$. We distinguish three subcases.

\medskip
\noindent\textbf{Case 2a: $n\ge T$.} 
Then $\min\{n, T\} = T$, so $A = \max\{T, s\} = T$ (since $s < T$). 
For $B$, we note $n+s \ge T$. Also, $\sqrt{TnL} \ge \sqrt{T^2 L} = T\sqrt{L}$. 
Since $n \ge 2$, $\sqrt{L} \ge \sqrt{\log 2} > \frac{1}{2}$, so $\sqrt{n}s \ge \frac{1}{2}T$. 
Therefore, $B = \min(T, n+s, \sqrt{n}s)$ satisfies $\frac{1}{2}T \le B \le T$, or $\frac{1}{2}A \le B \le A$.

\medskip
\noindent\textbf{Case 2b: $n<T$ and $n\ge s$.} 
Then $\min\{n, T\} = n$, so $A = \max\{n, s\} = n$. 
Since $s \le n$, we have $B \le n+s \le 2n = 2A$. 
For the lower bound, $n+s \ge n$ and $\sqrt{TnL} \ge n\sqrt{L} \ge \frac{1}{2}n$. 
Thus $B \ge \frac{1}{2}n = \frac{1}{2}A$.

\medskip
\noindent\textbf{Case 2c: $n<T$ and $n<s$.} 
Then $\min\{n, T\} = n$, so $A = \max\{n, s\} = s$. 
Since $n < s$, $B \le n+s \le 2s = 2A$. 
For the lower bound, $T > s$, $n+s > s$, and $\sqrt{n}s \ge s$ (since $n \ge 2 \implies \sqrt{n} > 1$). 
Thus every term in the minimum is $\ge s$, so $B \ge s = A$.

\medskip
Therefore, in all cases, $\frac12 A \le B \le 2A$.

\end{proof}


\end{document}